\documentclass[12pt]{article}
\usepackage{amsmath}
\usepackage{graphicx}
\usepackage{natbib}
\usepackage{url} 

\newcommand{\blind}{0}
\usepackage{lscape}

\usepackage{microtype}
\usepackage{subfigure}
\usepackage{booktabs} 
\usepackage{hyperref}

\usepackage{bm}
\usepackage{amssymb}
\usepackage{mathtools}
\usepackage{amsthm}

\usepackage[capitalize,noabbrev]{cleveref}
\RequirePackage{algorithm}
\RequirePackage{algorithmic}

\usepackage{multirow}
\usepackage{color}
\usepackage[usenames,dvipsnames]{xcolor}
\usepackage{mathtools}
\usepackage{array}
\newcolumntype{P}[1]{>{\centering\arraybackslash}p{#1}}
\usepackage{pbox}

\theoremstyle{plain}
\newtheorem{theorem}{Theorem}[section]

\newtheorem{lemma}[theorem]{Lemma}

\theoremstyle{definition}
\newtheorem{definition}[theorem]{Definition}

\theoremstyle{remark}
\newtheorem{remark}[theorem]{Remark}

\begin{document}

\def\spacingset#1{\renewcommand{\baselinestretch}%
{#1}\small\normalsize} \spacingset{1}

\date{}

\if0\blind
{
  \title{\bf Semi-Structured Distributional Regression --\\ Extending Structured Additive Models by Arbitrary Deep Neural Networks and Data Modalities}
  \author{David R\"ugamer\hspace{.2cm}\\
    Institute of Statistics, RWTH Aachen\\
    Department of Statistics, LMU Munich\medskip\medskip \\ 
    Chris Kolb\hspace{.2cm}\\
    Department of Statistics, LMU Munich\medskip\\
    and\medskip \\
    Nadja Klein\hspace{.2cm}\\
    Chair of Statistics and Data Science, Humboldt-Universit\"at zu Berlin
    }
  \maketitle
} \fi

\if1\blind
{
  \bigskip
  \bigskip
  \bigskip
  \begin{center}
    {\LARGE\bf Semi-Structured Distributional Regression --\\ a Statistical Approach to Multimodal Deep Learning}
\end{center}
  \medskip
} \fi

\bigskip
\begin{abstract}
Combining additive models and neural networks allows to broaden the scope of statistical regression and extend deep learning-based approaches by interpretable structured additive predictors at the same time. Existing attempts uniting the two modeling approaches are, however, limited to very specific combinations and, more importantly, involve an identifiability issue. As a consequence, interpretability and stable estimation are typically lost. We propose a general framework to combine structured regression models and deep neural networks into a unifying network architecture. To overcome the inherent identifiability issues between different model parts, we construct an orthogonalization cell that projects the deep neural network into the orthogonal complement of the statistical model predictor. This enables proper estimation of structured model parts and thereby interpretability. We demonstrate the framework's efficacy in numerical experiments and illustrate its special merits in benchmarks and real-world applications.
\end{abstract}

\noindent%
\vfill

\newpage
\spacingset{1.5} 

\section{Introduction}

In many applications, it is crucial to also capture uncertainties of predictions besides mere point estimates. 
Different approaches exist to model this (aleatoric) uncertainty and estimate predictive distributions; which we refer to as distributional regression (DR). Examples are quantile regression \citep{Koenker.2005} or structured additive distributional regression~\citep[SADR; see, e.g.,][]{Rigby.2005,Klein.2015}. SADR explicitly models all distributional parameters of a distribution $\mathcal{D}$ and thereby allows to learn the whole distribution. The general idea of modeling distributions has been advocated early in the machine and deep learning literature \citep[e.g., density networks as proposed by][]{Bishop.1994}. Among the approaches that directly model predictive distributions, DR can be considered a natural extension of traditional statistical regression models. Classical mean regression is used to estimate the expectation $\mu := \mathbb{E}_{\mathcal{D}}(Y|\boldsymbol{\nu})$ of a target variable $Y$ using tabular input features $\bm{\nu}$. The most widespread mean regression is commonly known as the generalized linear model \citep[GLM;][]{Nelder.1972}, defining the set of candidate models as 
$$ \mu = \mathbb{E}_{\mathcal{D}}(Y|\boldsymbol{\nu}) = h(\bm{\nu}^\top \bm{w})$$ 
with model weights or coefficients $\bm{w}$ and response (or activation) function $h$. As an extension, generalized additive models~\citep[GAMs;][]{Wood.2017} allow for general structured additive predictors
$$\mu = h(\textstyle \sum_j f_j(\bm{\nu}))$$
that allow estimating functional (e.g.~non-linear, spatial) relationships  between features and the response through generic function space representations $f_j$. GLMs and GAMs usually assume only very few and low-dimensional interactions between features $\bm{\nu}$ in the additive predictor $\sum_j f_j(\bm{\nu})$. Due to the additivity and function space assumptions of feature effects, this so-called \emph{structured predictor} allows for a straightforward interpretation of the model components. For example, in a GAM with no interactions, 
$$\mu = h(\textstyle \sum_j f_j(\nu_j)),$$
the contribution of each effect $f_j(\nu_j)$ on $h^{-1}(\mu)$ can be quantified explicitly by holding all other effects in the additive predictor constant. SADR extends GAMs by combining the idea of DR and structured predictors. SADR thereby relates various, potentially different additive predictors to each distributional parameter of an arbitrary parametric distribution. This allows the practitioner to explicitly model the aleatoric uncertainty of the data generating processes. As in GAMs, the interpretation of feature effects on each of the distributional parameters, e.g., the non-linear effect of one feature on the variance of a normal distribution, is straightforward. Yet, structured additive models are limited in their application (mainly tabular data) and their simplistic model structure (e.g., not allowing higher-order interactions).



In this work we extend SADR to a more flexible approach we call \emph{semi-structured distributional regression}, combining the structured additive predictor with an \emph{unstructured} model part learned through a deep neural network (DNN). Our goal is a distributional modeling framework that covers classical deep learning (DL) applications, use cases requiring higher-order feature interactions or multimodal learning settings, while still preserving the interpretability of (structured parts of) the model. This framework can describe the entire conditional distribution of the response rather than only the mean (and thus point predictions), but also includes mean regression as a special case. Existing approaches in DL that fuse structured predictors for tabular data with other data modalities in a multimodal network are sometimes referred to as \emph{wide and deep neural networks}. Several approaches exist \citep[see, e.g.,][for two recent applications]{Cheng.2016, Poelsterl.2020}, but these wide and deep models only focus on mean regression and linear effects. Our proposal can also be seen as an extension of mean regression wide and deep models to distributional wide and deep networks. We term the combination semi-structured as the statistical regression predictor is always structured, whereas the DNN-based predictor is considered unstructured. Typically, the structured model part processes tabular feature information whereas the unstructured part can in addition also take other data formats such as images or texts into account.

\subsection{Related Work} 

Existing SADR models in statistics \citep[e.g.,][]{Klein.2015, Groll.2019} assume a complex structured additive predictor. Several authors have described amalgamations of a simpler statistical model and a neural network (NN). \citet{Sarle.1994} describes the commonalities and differences between statistical models and NNs, \citet{De.2011, Bras.2019} consider GAMs when framed as a NN. \citet{Agarwal.2020} propose learning the non-linear additive functions of GAMs within a NN using separate networks for each feature. Various authors in the statistical literature have also proposed a combination of statistical regression and NNs. \citet{Tran.2018} propose the class of deep GLMs, where only the conditional expectation $\mu$ is modeled through a predictor that is based on a DNN. Their model is similar to the deep Bayesian regression model of \citet{Hubin.2018} who implement exact Bayesian inference, while \citet{Tran.2018} consider an approximation via fixed form variational Bayes and allow for additional random effects.  In \citet{Umlauf.2018} a single-layer NN is included in a SADR model for a very specific choice of weight and bias generation, but this approach does not allow for more complex or deep network structures. Existing wide and deep NNs only focus on modeling the distribution mean \citep[e.g.,][]{Cheng.2016, Chen.2018} or the hazard in survival regression \citep{Poelsterl.2020}. \citet{Li.2021} propose a deep distributional regression approach  by transforming the estimation problem into a constrained multi-class classification problem but do not use (semi-)structured additive predictors.

\subsection{Simplified Problem Formulation} 
While approaches marrying statistical regression and DL exist, these methods either do not provide aleatoric uncertainty or are limited to special use cases. More importantly, none of the existing approaches that try to leverage interpretable regression models and DL actually preserve the interpretability of the structured regression part satisfactorily. This is due to an identifiability problem between the structured and unstructured predictors.
For instance, consider a deep GLM that fuses a DNN $d$ with inputs $\bm{\nu}$ additively with a linear GLM predictor using the same features: $$\eta=h^{-1 }(\mu) = \bm{\nu}^\top \bm{w} + d(\bm{\nu}).$$ From the universal approximation theorem \citep{Cybenko.1989} and related literature, it is well known that $d$ can potentially approximate any continuous function under certain conditions. Thus, without loss of generality, assume that $d(\bm{\nu}) = \bm{\nu}^\top \bm{m} + f(\bm{\nu})$, decoupling the learned effect of $d$ into a linear predictor part $\bm{\nu}^\top \bm{m}$ and some non-linear effect $f$ of $\bm{\nu}$. We can directly observe that the deep GLM encompasses an identifiability issue, as we can arbitrarily shift a linear portion of $\bm{\nu}$ from the linear model part in ${\eta}$ to the deep part $d$ of the model and vice versa, e.g.: $$\bm{\nu}^\top \bm{w} + d(\bm{\nu}) = \bm{\nu}^\top \bm{w} + \bm{\nu}^\top\bm{m} + f(\bm{\nu}) = \bm{\nu}^\top \bm{m} + \tilde{d}(\bm{\nu}).$$ While not relevant for prediction, this identifiability issue inhibits interpretability of the model. In the deep GLM example, it is unclear how much of the linear effect the model will be attributed to the structured predictor and how much to the deep model part. This problem becomes even more entangled for more complex structured predictors as commonly used in GAMs or SADR.

\subsection{Main Contributions} 


We present a novel neural network-based framework for the combination of SADR and (deep) NNs. We address challenges with estimation and tuning of such a model, and in particular, propose a solution to the inherent identifiability issue in this model class based on a well-chosen \emph{orthogonalization cell}. This cell permits the joint modeling of the structured model part and the unstructured DNN predictor in a unifying, end-to-end trainable network, while preserving the interpretability of the structured part. 
On the one hand, our model class subsumes classical statistical regression models such as GAMs or SADR as special cases within a NN architecture
while having similar or even superior estimation performance. 
On the other hand, due to the generality of our approach, extensions of existing regression approaches can be built in a straightforward fashion. This facilitates, e.g., combining interpretable structured effects of tabular features with a DNN to capture potential higher-order interactions, or multimodal learning problems in which multiple different data modalities such as tabular, image or text data are present.

After introducing semi-structured distributional regression and addressing the accompanied identifiability issue in Section \ref{distreg}, we describe implementation details in Section~\ref{architect}. In Section~\ref{sec:sim}, we first examine how to achieve the identifiability of our model. Thereafter, we investigate its estimation and prediction performance in comparison with state-of-the-art distributional models. We then apply our method to several benchmark data sets in Section~\ref{application}. Finally, we will use the proposed framework in a multimodal data setting to predict prices of rental apartments in Section~\ref{sec:airbnb}.  We implement our proposed framework in a corresponding software package, available at \url{https://github.com/davidruegamer/deepregression}. All codes to reproduce results from numerical experiments, benchmarks and our application are available at \url{https://github.com/davidruegamer/semi-structured_distributional_regression}.


\section{Semi-Structured Distributional Regression} \label{distreg}

In this section, we first introduce SADR  and define the basic notation. Afterward, we provide the theoretical basis to achieve identifiability in our general model class.

\subsection{Distributional Regression and Notation} 

As briefly introduced before, SADR models aim at estimating arbitrary parametric distributions $\mathcal{D}(\theta_1,\ldots,\theta_K)$ by learning the corresponding distributional parameters $\bm{\theta} = (\theta_1,\ldots,\theta_K)^\top\in\Theta\subseteq\mathbb{R}^{K}$. SADR allows to regress features $\boldsymbol{\nu}$ on potentially all parameters $\theta_k,\,k=1,\ldots,K$ of the response distribution $\mathcal{D}$. In the spirit of GAMs, each of the $K$ distributional parameters is related to (possibly different subsets of) available features $\boldsymbol{\nu}$ through a monotonic and differentiable response function $$\theta_k(\bm{\nu}) = h_k(\eta_k(\bm{\nu})).$$ 
The predictors $\eta_k(\bm{\nu})\in\mathbb{R}$ specify the relationship between features $\boldsymbol{\nu}$ and the (transformed) parameters $h_k^{-1}(\theta_k)$, while 
$h_k$ ensures that possible parameter space restrictions on $\theta_k$ are fulfilled, e.g. $h_k(\cdot)=\exp(\cdot)$ to ensure positivity for a variance parameter. 
Note that $\mathcal{D}$ itself can be also a more complicated distribution, e.g., a mixture of distributions.

\subsection{Semi-Structured Distributional Regression}

Our proposed framework advances SADR by extending the additive predictors $\eta_k$ to include one or more latent representations learned through (possibly different) DNNs. To this end we embed SADR into a NN and learn the distribution by adapting deep probabilistic modeling approaches. We denote our approach detailed in the following \emph{semi-structured distributional regression} (SSDR).

\subsubsection{Output Model Structure} \label{respmodstruct}

In order to implement SADR in a NN, we define the last layer of the network as a distributional layer that computes the predicted distribution $\mathcal{D}({\theta_1}(\bm{\nu}),\ldots,\theta_K(\bm{\nu}))$ based on the outputs $\eta_k(\bm{\nu})$ of the subnetworks for the $\theta_k, k = 1, \ldots, K$. Given a realization $y$ of $Y$, the model can be estimated by optimizing the negative log-likelihood (NLL) $- \ell(\bm{\theta}) = -\log p_{\mathcal{D}}(y|{\theta_1}(\bm{\nu}),\ldots,\theta_K(\bm{\nu}))$ based on the probability density or mass function $p_{\mathcal{D}}$ of $\mathcal{D}$ evaluated at the estimated parameters $\hat{\bm{\theta}}$. For simplicity, we consider a univariate response $Y$, but the generalization 
to multivariate responses is readily available. 

\subsubsection{Network Inputs}
The subnetworks for the distributional parameters each process the  feature vector (or different subsets of it) $\boldsymbol{\nu}$. We consider $\bm{\nu}$ to be the concatenated set of input features $\bm{x}=(x_1,\ldots, x_p) \in \mathbb{R}^p$ modeled as structured linear effects, input features $\bm{z}= (z_1,\ldots, z_r) \in \mathbb{R}^r$ modeled as structured non-linear effects, and features $\bm{u}=(u_1,\ldots,u_q)$ that are passed through a DNN and constitute the unstructured model part in one or more of the additive predictor(s) $\eta_k, k=1,\ldots,K$. 

\subsubsection{Predictor Structure}
With this notation, the semi-structured predictors $\eta_k$ 
in SSDR are assumed to be an additive decomposition of structured linear parts $f_{k,0}(\bm{x})=b_k+\bm{x}^\top\bm{w}_{k}$, structured non-linear functions $f_{k,j}(z_{j})$, and one or more unstructured predictors $d_{k,j}(\bm{u})$ constituting linear combinations of latent features learned through DNNs. The inputs for these DNNs can be subsets of the features $\bm{u}$ and can also be (partially) identical to the features $\bm{x},\bm{z}$:
$$\eta_k = f_{k,0}(\bm{x})+ \textstyle\sum_{j=1}^{r_k} f_{k,j}(z_j) + \textstyle\sum_{j=1}^{g_k}  d_{k,j}(\bm{u}).$$ Note that we have suppressed the index $k$ in the subset of feature inputs to not overload notation. 
However, SSDR allows to specify all additive terms individually for all parameters $\theta_k$. This can be relevant for example when prior information is available about which features are relevant for the location or the scale of the response.

We further assume that the DNN predictors can be represented as $d_{k,j}(\bm{u}) = \hat{\bm{u}}_{k,j}^\top \bm{\gamma}_{k,j}$, with latent features $\hat{\bm{u}}_{k,j}$ taken as the outputs of the network trunk (all DNN layers up to the penultimate layer) and $\bm{\gamma}_{k,j}$ being the last-layer weights forming the network head (cf.~Figure~\ref{nn_arch}). The functions $f_{k,j}(\cdot)$ represent penalized smooth non-linear effects of univariate or low-dimensional features using a linear combination of $L_{k,j}$ appropriate basis functions. A univariate non-linear effect of feature $z_j$ is, e.g., approximated by $f_{k,j}(z_j) \approx \sum_{l=1}^{L_{k,j}} {B}_{k,j,l}(z_j) {w}_{k,j,l}$, where ${B}_{k,j,l}(z_j)$ is the $l$th basis function (such as regression splines, polynomial bases or B-splines) evaluated at $z_j$. Tensor product representations allow for two- or moderate-dimensional non-linear interactions of a subset of $\bm{z}$. It is also possible to represent discrete spatial information or cluster-specific effects in this way~\citep[see, e.g.,][]{Wood.2017}. This representation also allows for random effects in the additive predictor, as these can represented by ridge-penalized linear effects (see Section~\ref{pen_and_prior} for details).

\subsection{Identifiability}\label{sec:ortho}

Identifiability is crucial when features overlap in the structured and unstructured model parts. 
More formally, we here address the following identifiability issue for DR with structured additive predictor(s) $\eta^{str}_k$ and unstructured additive predictor(s) $\eta_k^{unstr}$, $k=1,\ldots,K$:
\begin{definition}{\textbf{Semi-structured identifiability}}\label{def:ident}
We say that a semi-structured distributional regression is identified in its structured model parts if there exists no $$\xi_k \neq 0: \eta_k = \eta_k^{str} + \eta_k^{unstr} = (\eta_k^{str} - \xi_k) + (\eta_k^{unstr} + \xi_k) = \breve{\eta}_k^{str} + \breve{\eta}_k^{unstr} = \breve{\eta}_k,$$ where $\ell(\breve{\bm{\theta}}) = \ell(\bm{\theta})$ and $\theta_k(\eta_k)$ in $\breve{\bm{\theta}}$ is replaced by $\theta_k(\breve{\eta}_k)$ for any $k\in\{1,\ldots,K\}$.
\end{definition}
To derive a theoretical concept assuring identifiability in the following, we assume w.l.o.g.~that we are only interested in linear effects of pre-specified features $\bm{x}$ and allow for a single additional DNN predictor $d_k$ with arbitrary features $\bm{u}$ in $\eta_k$ for one distributional parameter $\theta_k$. The treatment of the more general case can be found in Supplementary Material~A. We again suppress the index $k$ in the feature vectors and design matrices for readability. As we will explain in a later remark, we consider identifiability of every additive predictor $\eta_k$ on the level of $n$ observations. Therefore, let $(\bm{x}_1^\top, \ldots, \bm{x}_n^\top)^\top =: \bm{X} \in \mathbb{R}^{n \times p}$, $1 \leq p\leq n$, be the structured feature matrix for $n \in \mathbb{N}$ observations and $\bm{w}_k \in \mathbb{R}^p$ the corresponding weights. Further, define the collection of $n$ feature vectors $(\bm{u}_1^\top, \ldots, \bm{u}_n^\top)^\top$ as $\bm{U}\in\mathbb{R}^{n \times q}$, which is fed into the DNN $d_k$, and let $\widehat{\bm{U}}_k \in \mathbb{R}^{n \times s}$, denote the matrix collecting $n$  latent feature vectors of length $s, s \leq n,$ obtained from the second-last layer of $d_k$. Finally, let $\bm{\eta}_k =(\eta_{k,1},\ldots,\eta_{k,n})^\top\in \mathbb{R}^n$ be the final predictor vector.

If not constrained, $d_k$ is able to capture linear effects of any of its features $\bm{U}$, including those also present in $\bm{X}$, thus making the attribution of shared effects to either one of the model parts in $\bm{\eta}_k$ not identifiable. The following Lemma shows how the orthogonalization induces a meaningful decomposition of learned effects in $\bm{\eta}_k$ and guarantees identifiability in the sense of Definition~\ref{def:ident}.

\begin{lemma}{ \textbf{Orthogonalization}}\label{lem:1} Let  $\bm{\mathcal{P}}_X \in \mathbb{R}^{n \times n}$ the projection matrix for which $\bm{\mathcal{P}}_X \bm{A}$ is the linear projection of $\bm{A} \in \mathbb{R}^{n \times s}, s \leq n$, onto the column space spanned by the features of $\bm{X}$ and $\bm{\mathcal{P}}^\bot_X := \bm{I}_n - \bm{\mathcal{P}}_X$ the projection into the respective orthogonal complement. Then, replacing the latent features $\widehat{\bm{U}}_k$ with $$\widetilde{\bm{U}}_k = \bm{\mathcal{P}}^\bot_X \widehat{\bm{U}}_k$$ and multiplying the result with the last layer's weights $\bm{\gamma}_k \in \mathbb{R}^s$, ensures a decomposition of the final predictor \begin{equation} \label{thetak}
\bm{\eta}_k = \bm{X}{\bm{w}}_k + \widetilde{\bm{U}}_k \bm{\gamma}_k
\end{equation} 
into an identified linear part learned from features $\bm{X}$ and a non-linear part learned from features $\widetilde{\bm{U}}_k$.
\end{lemma}
A proof of Lemma~\ref{lem:1} can be found in the Supplementary Material~A and can be used to proof the identifiability of structured terms in $\bm{\eta}_k$ as follows. 
\begin{theorem}{ \textbf{Identifiability}}\label{theo1} Replacing $\widehat{\bm{U}}_k$ with $\widetilde{\bm{U}}_k$ and multiplying the result with the weights $\bm{\gamma}_k$ ensures identifiability of the structured linear part in the final predictor \eqref{thetak}.
\end{theorem}
A proof is given in Supplementary Material~A.

\begin{remark} 
If $\bm{x}$, $\bm{z}$ and $\bm{u}$ overlap in their features, we first reparameterize the structured non-linear model part to ensure identifiability between the structured linear and non-linear parts and then combine the two into one joint predictor (not shown in Figure~\ref{nn_arch}). The non-linear part can then be interpreted as the non-linear deviation from the corresponding linear effect. Finally, we apply the orthogonalization using the joint structured predictor to separate it from the unstructured DNN predictor.
\end{remark}

\begin{remark} 
When the DNN and the structured part share $p$ columns with $p>n$,  $\bm{\mathcal{P}}^\bot_X$ is equal to a zero matrix $\bm{0}_{n \times n}$, making $\widehat{\bm{U}}_k$ de facto irrelevant. In this case the estimated model is equivalent to the defined structured model. We see this edge case rather as a property than a limitation of our framework as the key requirement of the proposed architecture is to provide identifiable structured effects, which in this case is only possible by excluding the unstructured predictor part.
\end{remark}

\begin{remark}
Last, we note that an alternative orthogonalization type and commonly used technique to make effects identifiable in structured models \citep[see, e.g.,][]{Ruegamer.2018} is to calculate $\bm{\Xi}_k$, a matrix with columns spanned by $\operatorname{ker}(\hat{\bm{U}}_k^\top \bm{X})$, and then use $\widetilde{\bm{U}}_k = \hat{\bm{U}}_k \bm{\Xi}_k$. This would have the advantage of being applied on the columns of $\hat{\bm{U}}_k$ and not its rows and would therefore make the orthogonalization independent of the sample or batch size. However, as $\hat{\bm{U}}_k$ is updated in each iteration, the required null space $\operatorname{ker}(\hat{\bm{U}}_k^\top \bm{X})$ is also changing and can thus not be used in an end-to-end differentiable architecture (as gradients cannot be calculated for a random null space). This is the reason we chose the orthogonalization as proposed in Lemma~\ref{lem:1} and consider identifiability on the level of observations.
\end{remark}

\subsection{Unifying Network Architecture}

In contrast to most previous approaches that combine structured regression and DNNs, we propose a unifying network that learns the structured effects of $\bm{x}$ and $\bm{z}$ using single unit hidden layers with linear activation functions and different regularization terms for each input type and each distributional parameter. The DNN model part(s) processing $\bm{u}$ can be arbitrarily specified to, e.g., incorporate complex feature interactions. To ensure identifiability, we propose an orthogonalization cell based on the previous theorem. This allows us to combine structured and unstructured feature effects in a distributional regression setting while ensuring identifiability. Figure~\ref{nn_arch} and its caption provide a detailed explanation.
 \begin{figure}[!h]
 \vskip 0.2in
 \begin{center}
 \small
 \centerline{\includegraphics[width=0.5\textwidth]{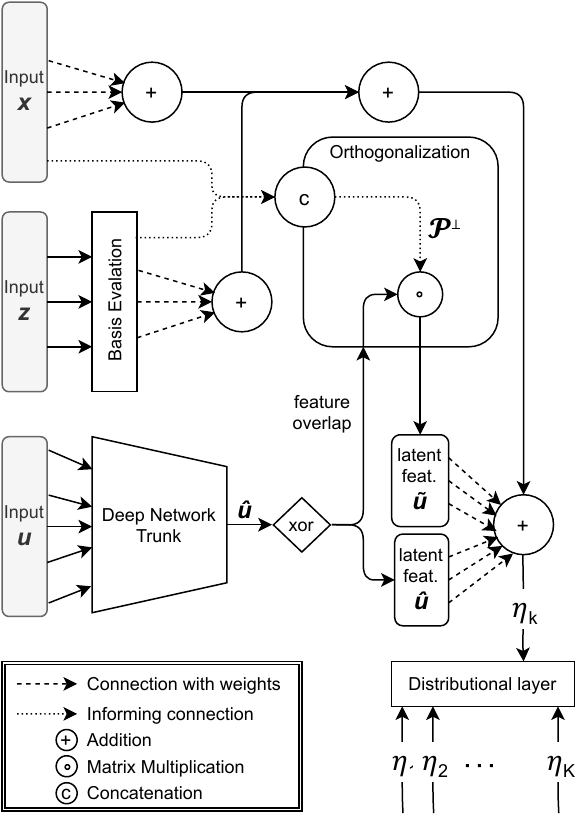}} 
 \caption{Exemplary SSDR architecture incorporating an orthogonalization cell: Structured linear features $\bm{x}$ and structured non-linear features $\bm{z}$ (represented by  their respective basis evaluations) are fed into the cell, concatenated and combined with the DNN using the orthogonalization operation. The solid lines represent the informing connections that either simply combine information or create a projection matrix $\mathcal{P}^{\bot}$. The latter is multiplied with the latent features $\hat{\bm{u}}$ to disentangle the structured and unstructured parts. The final additive predictors $\eta_k$ are created by adding the structured predictors to the (orthogonalized) linear combination of latent features $\tilde{\bm{u}}$ and their weights before being passed on to the distributional layer.}
 \label{nn_arch}
 \end{center}
\vskip -0.2in
\end{figure}
%
\begin{remark}
For latent features $\hat{\bm{u}}_{k,j}$ learned in DNNs $d_{k,j}$, a distinction is made between DNNs whose inputs are also part of the structured inputs $\bm{x},\bm{z}$ and DNNs whose inputs only appear in the unstructured predictor (xor-node in Figure~\ref{nn_arch}). In the latter case, the DNN outputs are directly summed up as a weighted combination with structured predictors (lower path after xor-node in Figure~\ref{nn_arch}) and fed into the distributional layer. For those $\hat{\bm{u}}_{k,j}$ whose DNNs $d_{k,j}$ also share inputs with $\bm{x}$ or $\bm{z}$, the orthogonalization operation is applied before adding its outputs $\tilde{\bm{u}}_{k,j}$ as weighted sums to the remaining predictor parts. While this approach ensures the identifiability of the structured model part, a custom orthogonalization for specific (non-overlapping) inputs in $\bm{u}$ might also be interesting in applications (see Section~\ref{sec:conclusion}).
\end{remark}

\section{Implementation Details} \label{architect}

We now provide further details on various implementation aspects of the SSDR framework. 

\subsection{Penalization, Optimization and Tuning} \label{pen_and_prior}

It is common ground that gradient descent (GD) routines used to train NNs hold an implicit regularization behaviour \citep[see, e.g.,][]{Arora.2019}. However, in most of our experiments we observe rather coarse estimated non-linear effects or even convergence difficulties when not penalizing structured non-linear effects. We therefore allow for additional quadratic penalization of structured non-linear effects 
by regularizing the corresponding weights, with the regularization strength controlled through smoothing parameters $\lambda_{k,j}\in\mathbb{R}$ for each structured non-linear term $f_{k,j}, k \in \{1, \ldots, K\}, j \in\lbrace 1,\ldots, r_k\rbrace$ and appropriate penalty matrices $\bm{S}_{k,j} \in \mathbb{R}^{L_{k,j} \times L_{k,j}}$ \citep[see, e.g.,][]{Wood.2017}. Optimization of the model is done by minimizing the corresponding penalized NLL of  $\mathcal{D}$
$$-\log p_{\mathcal{D}}(y|\hat{\bm{\theta}}(\bm{\nu})) +  
\textstyle \sum_{k = 1}^K  \textstyle \sum_{j=1}^{r_k} \lambda_{k,j} \bm{w}_{k,j}^\top \bm{S}_{k,j} \bm{w}_{k,j}. $$
Tuning a large number of smooth terms, potentially in combination with DNNs, is a challenging task. Non-linear structured additive models alone require a sophisticated estimation procedure, often based on second-order information and full batch training. This gets even more complicated in semi-structured models.  As DL platforms allow for custom optimization routines, one possible optimization strategy could be an iterative procedure alternating between the structured and unstructured model parts. This would allow the use of a stochastic GD routine for the DNN part(s), and a classical statistical optimization for the structured model part(s) including the estimation of smoothness parameters. Here however, we propose an alternative strategy that fosters easy training and tuning by defining the smoothness of each effect in terms of the degrees of freedom \citep[$df$; see, e.g.,][]{Buja.1989}. This approach can be implemented efficiently using the Demmler-Reinsch Orthogonalization \citep[DRO, cf.][]{Ruppert.2003}. The latter can easily be parallelized or sped up using a randomized singular value matrix decomposition \citep{Erichson.2016}. We thereby also allow for meaningful default penalization and comparability of effect complexities by setting $df_{k,j}$ to the same value $df_k^\ast$ for all smooth effects $j$ of the $k$th parameter. At the same time, we ensure enough flexibility by choosing $df_k^\ast = \min_{j \in\lbrace 1,\ldots, r_k\rbrace} \max df_{k,j}$, i.e., the largest possible degree of freedom that leaves the least flexible smoothing effect unpenalized and regularizes all others to have the same amount of flexibility. After estimating the smoothing parameters $\lambda_{k,j}\in\mathbb{R}$, optimization can be done efficiently and jointly for all model components using a stochastic GD routine.



\subsection{Implementation of the Orthogonalization} \label{sec:oz}

The orthogonalization proposed in the previous section can be implemented in different variants, depending on the batch size used in training. When using full-batch GD, the projection $\bm{\mathcal{P}}^\bot_X$ can be calculated upfront and applied in each epoch, yielding exact orthogonalization. For mini-batch training with batch matrices $\bm{X}_b$, we calculate the respective projections for each batch using, e.g., a QR decomposition $\bm{\mathcal{P}}^\bot_{X_b} = (\bm{I}-\bm{Q}_b\bm{Q}_b^\top)$ with $\bm{X}_b = \bm{Q}_b\bm{R}_b$ that allows for stable calculation and inclusion in fully automatic differentiated routines \citep[see, e.g.,][]{Roberts.2020}. In the next section, we will investigate the difference between full-batch GD and its stochastic GD variant using mini-batch orthogonalization.

\section{Numerical Experiments} \label{sec:sim}

 \begin{figure*}[tb]
 \vskip 0.2in
 \begin{center}
 \centerline{\includegraphics[width=\textwidth]{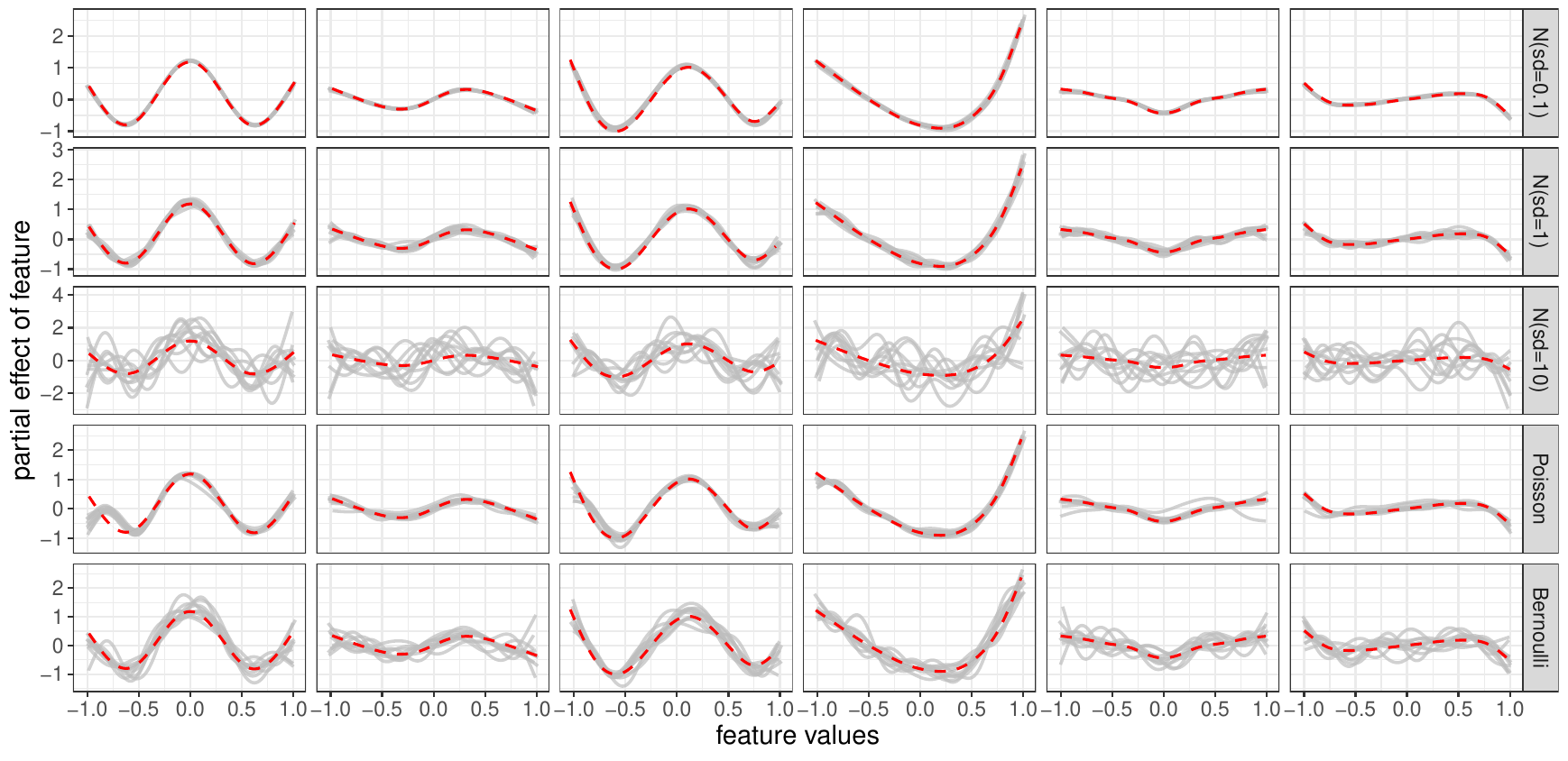}} 
 \caption{Non-linear partial effects of six selected features (columns) on the mean of the response from the five different distributions (rows) with true effect in red and estimated functions of the 10 replicates in grey.}
 \label{vardecomp}
 \end{center}
\vskip -0.2in
\end{figure*}

We conduct three different simulation experiments to assess the goodness-of-fit in terms of the structured effect estimation and the prediction performance. Our first experiment demonstrates the efficacy of the orthogonalization cell in practice. In the second experiment, we examine the properties of the orthogonalization operation when using mini-batch stochastic GD training. The third experiment compares the proposed framework against classical statistical regression frameworks to demonstrate that estimating a penalized SADR model works equally well when cast as a NN. Unless stated otherwise, the structured non-linear effects $f_{k,j}$ in all experiments and benchmarks are instantiated using thin-plate regression splines \citep[see, e.g.,][]{Wood.2017}.

\subsection{Identifiability and Interpretability} \label{sec:simident}

Here, we mimic a situation where the  practitioner's interest explicitly lies in decomposing certain feature effects into structured linear, structured non-linear and unstructured non-linear parts, for reasons of interpretability. We simulate 10 data set replicates with $n=1500$ observations and $p=10$ features 
drawn independently from a uniform distribution $\mathcal{U}(-1,1)$. For the response $Y$ we consider the cases $Y\sim\mathcal{N}(\eta,s^2)$, $s=0.1,1,10$ (Normal), $y\sim\textit{Ber}(\text{sigmoid}(\eta))$ (Bernoulli) and $y\sim\textit{Po}(\exp(\eta))$ (Poisson). The predictor $\eta$ ($K=1$) contains linear and non-linear effects of all features, as well as an interaction term of the 10 features: $$
\eta = \textstyle b + \sum_{j=1}^{10} x_jw_j + \textstyle\sum_{j=1}^{10} f_j(x_j) + \log_{10}(\textstyle\prod_{j=1}^{10} (x_j+2)).
$$
The weights $w_1, \ldots, w_{10}$ are defined as $\frac{2}{1}, \frac{2}{2}, \ldots, \frac{2}{10}$. $f_1, \ldots, f_{10}$ are ten different non-linear functions defined in the Supplementary Material~B. Our framework explicitly models the true linear and non-linear terms by separating both effects via orthogonalization. Further, to account for the interaction, an unstructured DNN predictor is defined using a fully connected network with ReLu activation and 32 and 16 hidden units in two hidden layers, respectively. By projecting the output of the second-last layer into the orthogonal complement of the structured predictors, 
we ensure identifiability of the linear and non-linear effects. We do not fine-tune the model but rather use the DRO approach described in Section \ref{pen_and_prior} and train for a fixed number of $2000$ epochs. All models are optimized using the Adam optimizer \citep[][]{kingma2014adam} with a learning rate of 0.01 and a batch size of 32.

\emph{Results}. Figure \ref{vardecomp} visualizes the estimated and true non-linear relationships between selected features and the response. 
Overall the resulting estimates in Figure \ref{vardecomp} demonstrate the capability of the framework to recover the true partial non-linear effects, and only the simulation using a normal distribution with $\mbox{Var}(Y) = 100$, which amounts to a maximum of 4\% signal-to-noise ratio (predictor variance divided by noise variance), shows an overfitting behavior. Most importantly, the results highlight that the DNN predictor, which is also fed all 10 features, does not learn the linear or non-linear part of the structured effects, and thereby the structured part of the additive predictor $\eta$ remains identifiable and interpretable as constructed and desired.

\subsection{Mini-Batch Orthogonalization}

As discussed in Section~\ref{sec:oz}, a full-batch orthogonalization yields an exact projection, while various large-scale applications require mini-batch training, and hence only permit an approximate orthogonalization. Here, we provide evidence that the latter works equally well in practice. We therefore simulate linear models $\bm{Y} = \bm{X} \bm{w} + \bm{\varepsilon}$, where $\bm{X} \in \mathbb{R}^{n \times p}$ and $\bm{\varepsilon}\sim\mathcal{N}(\bm{0}_{n},\bm{I}_{n})$, with different number of samples $n \in \{100, 1000, 10000 \}$ and features of size $p \in \{1, 10\}$ drawn independently from $\mathcal{N}(0,1)$. We define $\bm{w}$ to be equally spaced coefficients from $-3$ to $3$ ($w=-3$ for $p=1$). We then check whether the true effects of these features can be recovered in the presence of a DNN (fully-connected NN with ReLU activation and 100 and 50 hidden units in two hidden layers). The DNN is provided the same input features as defined for the structured part using either no orthogonalization, full-batch or mini-batch orthogonalization with batch sizes $B \in \{ 25, 50 \}$. The models are trained for $1000$ epochs using Adam optimizer with early stopping and a learning rate of $0.001$. We repeat each experiment 10 times based on different realizations of standard normal errors $\bm{\varepsilon}$.

\emph{Results}. Figure \ref{fig:boxplotsOZ} compares the oracle linear model with the SSDR model with and without orthogonalization by looking at the squared deviations of the estimated model coefficients $\bm{w}$. Results show no notable difference between training with different batch sizes when using the orthogonalization. The figure also highlights the problem of identifiability when not orthogonalizing, yielding large deviations from the ground truth in all settings except for the ones with a large number of observations ($n=10000$) and mini-batch training.

\begin{figure}
    \centering
    \includegraphics[width=0.8\textwidth]{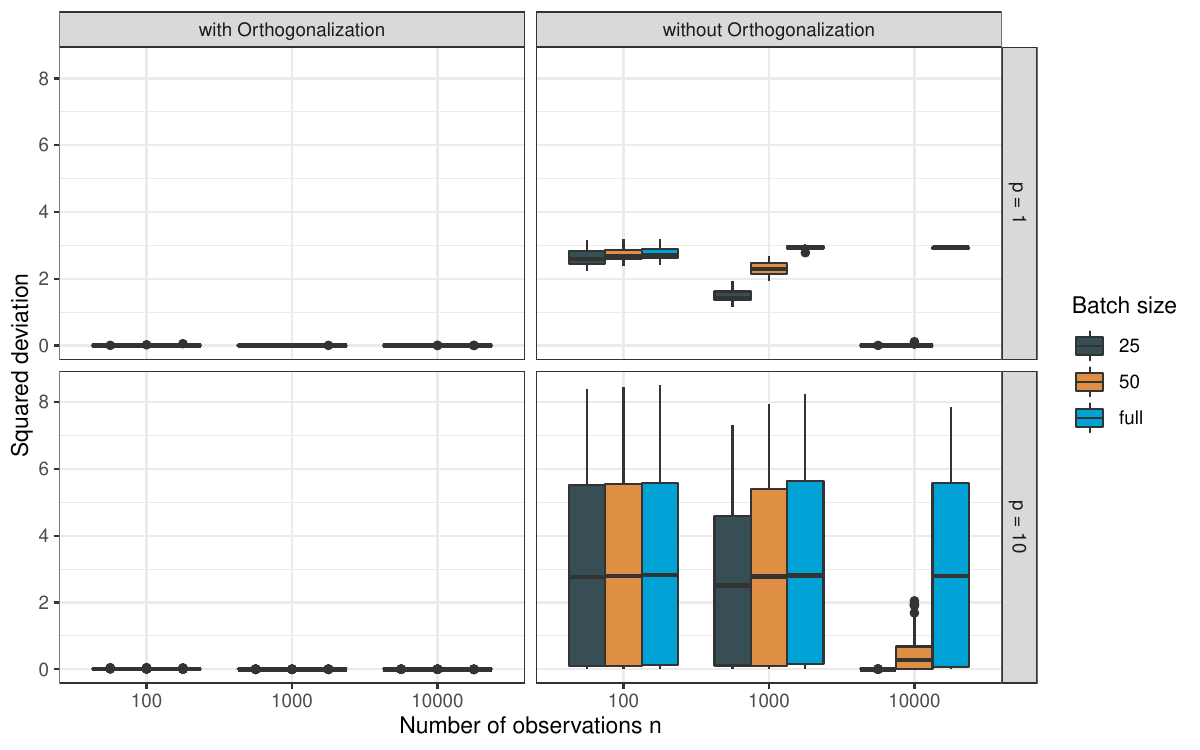}
    \caption{Squared deviations of model coefficients of the oracle model and the estimated model (y-axis) for different numbers of observations (x-axis), different numbers of features (rows), different batch sizes (colors) both with and without orthogonalization (columns). The differences in squared deviations for ``with orthogonalization'' are on the scale of $10^{-5}$ and thus negligible.}
    \label{fig:boxplotsOZ}
\end{figure}

\subsection{Model Comparison} \label{sec:simcomp}

In this simulation, we compare the estimation and prediction performance of our framework with three other approaches that implement SADR based on a location, scale and shape (LSS) parameterization. More specifically, we run our approach against a likelihood-based optimization \citep[gamlss;][]{Rigby.2005}, a Bayesian optimization \citep[bamlss][]{Umlauf.2018} and a model-based boosting routine \citep[MBB;][]{Mayr.2012}. 
For three different distributions (normal, gamma, logistic), we investigate a combination of different sample sizes $n\in \{300,2500\}$ and different numbers of linear feature effects in the location ($p \in \{10,75\})$, while fixing the number of linear effects in the scale parameter to two. In addition to the linear feature effects, we add 10 non-linear effects for the location and two non-linear effects for the scale. Further details can be found in Supplementary Material~B. The SSDR models are trained using Adam with a learning rate of $0.001$, batch size of 32, and the number of epochs selected by 5-fold cross-validation. The simulation results are replicated 20 times. 

\emph{Results}. Table \ref{tab:comparison} summarizes the mean log-scores of $0.25n$ test data points and mean RMSE measuring the average of the deviations between true and estimated structured effects. We observe that our approach (SSDR) often yields the best or second-best estimation and prediction performance while being robust to more extreme scenarios in which the number of observations is small in comparison to the number of feature effects. In this situation, other approaches tend to suffer from convergence problems. We conclude that the estimation of SADR within our DNN framework works at least as good as classical statistical approaches, but yields more robust results in high-dimensional settings.

\setlength\tabcolsep{1.5pt}
\begin{table*}[tb]
\begin{center}
\begin{scriptsize}
\begin{tabular}{P{8pt}P{8pt}P{8pt}|P{48pt}P{48pt}P{48pt}P{48pt}|P{48pt}P{48pt}P{48pt}P{48pt}}

&&&  \multicolumn{4}{c}{Negative Log-scores} & \multicolumn{4}{c}{RMSE} \\
  & $n$ & $p$ & bamlss & SSDR & gamlss & MBB &   bamlss & SSDR & gamlss & MBB \\ \hline
\multirow{4}{*}{\rotatebox[origin=c]{90}{Normal}} & \multirow{2}{*}{\rotatebox[origin=c]{90}{300}} & 10 & \textbf{1.51} (0.68) & 1.85 (0.55) & 7.08 (8.07) & 4.20 (0.47) &  0.89 (0.61) & \textbf{0.37} (0.28) & 0.59 (0.15) & 0.84 (0.22) \\ 
    &  & 75 &  $> 10e{20}$ & \textbf{2.84} (0.94) & $> 10e{20}$ & 10.4 (0.95) &   1.05 (1.14) & \textbf{0.47} (0.42) & 0.67 (0.26) & 1.29 (0.86) \\ 
  & \multirow{2}{*}{\rotatebox[origin=c]{90}{2500}} & 10 & \textbf{0.55} (0.06) & 0.96 (0.18) & 0.57 (0.08) & 3.71 (0.24) &   0.50 (0.71) & \textbf{0.22} (0.22) & 0.25 (0.34) & 0.66 (0.35) \\ 
    &  & 75 & \textbf{0.64} (0.06) & 1.11 (0.12) & 0.69 (0.07) & 8.85 (0.51) &   0.48 (0.70) & \textbf{0.19} (0.22) & 0.24 (0.32) & 1.14 (0.64) \\ 
   \multirow{4}{*}{\rotatebox[origin=c]{90}{Gamma}} & \multirow{2}{*}{\rotatebox[origin=c]{90}{300}} & 10 & 1.15 (0.10) & 1.32 (0.31) & \textbf{1.04} (0.09) & 1.13 (0.11) &  0.13 (0.06) & 0.14 (0.04) & \textbf{0.08} (0.02) & 0.11 (0.04) \\ 
    &  & 75 & \textbf{1.50} (0.36) & 2.34 (0.87) & 2.05 (0.99) & 1.56 (0.15) &  0.15 (0.06) & 0.18 (0.07) & \textbf{0.12} (0.05) & 0.14 (0.05) \\ 
    & \multirow{2}{*}{\rotatebox[origin=c]{90}{2500}} & 10 & 1.01 (0.02) & 0.93 (0.03) & \textbf{0.83} (0.02) & 0.96 (0.03) &   0.19 (0.18) & 0.10 (0.04) & \textbf{0.04} (0.03) & 0.10 (0.08) \\ 
 & & 75 &   1.01 (0.03) & 1.24 (0.05) & \textbf{0.84} (0.03) & 1.01 (0.04) &   0.19 (0.20) & 0.07 (0.03) & \textbf{0.04} (0.02) & 0.11 (0.07) \\ 
    \multirow{4}{*}{\rotatebox[origin=c]{90}{Logistic}} & \multirow{2}{*}{\rotatebox[origin=c]{90}{300}} & 10 & \textbf{1.44} (0.14) & 1.75 (0.12) & 4.38 (2.84) & 3.28 (0.37) &  1.69 (0.98) & \textbf{0.28} (0.14) & 0.61 (0.18) & 0.97 (0.36) \\ 
     &  & 75 & 2.55 (0.48) & \textbf{2.22} (0.18) & 124 (104) & 4.89 (0.26) &   1.82 (1.21) & \textbf{0.26} (0.15) & 0.66 (0.44) & 1.30 (0.84) \\ 
    & \multirow{2}{*}{\rotatebox[origin=c]{90}{2500}} & 10 & 1.7 (0.04) & \textbf{1.15} (0.06) & 1.15 (0.04) & $> 10e{20}$ &  2.20 (0.82) & \textbf{0.18} (0.15) & 0.24 (0.32) & 0.73 (0.41)  \\  
     &  & 75 & 2.47 (0.36) & \textbf{1.16} (0.08) & 1.23 (0.05) & 4.71 (0.1) & 
 3.52 (0.97) & \textbf{0.13} (0.13) & 0.25 (0.32) & 1.19 (0.66) \\ \hline
   \end{tabular}
\end{scriptsize}
\end{center}
\vskip -0.2in
 \caption{Median and median absolute deviation of the mean negative predictive log-scores and mean RMSE values of estimated weights and non-linear point estimates across all settings and 20  replications. The best performing approach is highlighted in bold.}
    \label{tab:comparison}
\end{table*}

\section{Benchmark Studies}\label{application}

In addition to the experiments on synthetic data, we provide a number of benchmarks in several real-world data sets in this section. If not stated otherwise, we report measures as averages and standard deviations (in brackets) over 20 random network initializations. Unless reported otherwise, activation functions for DNN predictors in SSDR are ReLU for hidden layers and linear for output layers. Additional investigation of our framework checking its competitiveness for quantile regression, calibrated regression, and high-dimensional classification problems can be found in Supplementary Material~C.

\subsection{Deep Mixed Models}\label{subsec:deepmixed}

\citet{Tran.2018} use a panel data set with 595 individuals and 4165 observations from \citet{Cornwell.1988} as an example for fitting deep mixed models by accounting for within subject correlation. Performance is measured in terms of within subject predictions of the log of wage for future time points. We follow their analysis by training the model on the years $t=1, \ldots, 5$ and predicting the years $t=6,7$. We use a normal distribution with constant variance and model the mean with the same NN predictor as implemented by \citet{Tran.2018}. However, instead of being part of the DNN predictor, the subject ID is included as a structured random effect $w_i$ for each individual $i$: $\mbox{log-wage}_{i,t} \sim \mathcal{N}(b + w_i + d_\mu(\bm{x}_{i,t}), \exp(d_\sigma(\bm{x}_{i,t})))$, with $\bm{x}_{i,t}$ being the 12 features also used in \citet{Tran.2018}, individual specific random effects $w_i \sim \mathcal{N}(0,\tau^2) \, \forall i$, and $d_\mu$, $d_\sigma$ two different fully-connected NNs with two hidden layers and five neurons each. The model is estimated with default settings of Adam, batch size of 32 and the number of epochs chosen by cross-validation.

\emph{Results}. Our approach yields an average MSE of $0.04$ ($0.005$) which makes our method competitive with the approach of \citet{Tran.2018}, who report an MSE of $0.05$ for the given data split.

\subsection{Deep Calibrated Regression}

Next, we use the data sets \textit{Diabetes}, \textit{Boston}, \textit{Airfoil}, and \textit{Forest Fire} analyzed by \citet{Song.2019} to benchmark our SSDR approach against the two \textit{post-hoc} calibration methods  isotonic regression~\citep[IR;][]{kuleshov18a} and the GP-Beta model \citep[GPB;][]{Song.2019} with 16 inducing points. The uncalibrated model for the latter two is a Gaussian process regression (GPR) which performed better than ordinary least squares and standard NNs in \citet{Song.2019}. SSDR directly models the parameters of a normal distribution. 
Here, we only fine-tune the specific structure for the predictors $\eta_{\mu}=\mu,\eta_{\sigma}=\log(\sigma)$, which consist of structured linear and non-linear effects, as well as a DNN for all features (see the Supplementary Material~C.2 for details about each model's specification). We split the data into 75\%  for training and 25\% for model evaluation, measured by negative log-scores. 

\emph{Results}. Table~\ref{uci} suggests that compared to other calibration techniques our method yields more stable results while permitting interpretable structured effects for features of interest. Even though we did not fine-tune the output distribution, SSDR performs as good as the benchmarks in terms of average negative~log-scores.

\begin{table}[htbp]
\begin{center}
\begin{small}
\begin{tabular}{lcccc}
  & SSDR & GPR & IR & GPB \\ \hline
  Diabetes & \textbf{5.33} (0.00)  & 5.35 (5.76)  & 5.71 (2.97)  & 5.33 (6.24) \\
  Boston  & 3.07 (0.11) & 2.79 (2.05) & 3.36 (5.19) & \textbf{2.70} (1.91) \\
Airfoil  & \textbf{3.11} (0.02) & 3.17 (6.82) & 3.29 (1.86) & 3.21 (4.70) \\
Forest F.  & 1.75 (0.01)  & 1.75 (7.09)  & \textbf{1.00} (1.94)  & 2.07 (9.25) \\
\hline
\end{tabular}
\end{small}
\end{center}
\vskip -0.2in
\caption{Comparison of negative~log-scores 
of different methods (columns) on four different UCI repository data sets (rows).}
\label{uci}
\end{table}

\section{Application to Multimodal Data} \label{sec:airbnb}

Finally, we present an application of our framework to Airbnb price listing data, available at \url{http://insideairbnb.com/get-the-data.html}. Various data modalities are included in the data, such as numeric variables for latitude and longitude, integer variables such as the number of bedrooms, textual information such as a room description, dates as well as an image of each property (size $200\times200\times3$), see Table \ref{tab:feat}. We will focus on apartments in Munich, Germany consisting of 3,504 observations and 73 features of mixed data modalities. Our goal is to predict the listing price of each apartment in an interpretable structured additive model, while also accounting for the room description and images. As the data set is relatively small and the information content in the images seems also not substantially decisive for the room price, we use this application to demonstrate a further property of the orthogonalization---its regularization effect. To this end, we compare the model using two different specifications. One model is specified with structured and unstructured predictors which are simply added up. The other specification regularizes the learned latent image effects by orthogonalizing them w.r.t.~all specified structured features. In other words, the second model subtracts the structured information from the learned image information in order to regularize the DNN. For model inspection and fine-tuning, we set aside 10\% of the data for testing and use 10\% of the resulting training data for early stopping. 

\emph{Model Specification} For $\mathcal{D}$, we use a log-normal distribution. $\mathcal{D}$ is parameterized by its location (mean) $\mu$ and scale parameter $\sigma$ with corresponding additive predictors $\eta_\mu$ and $\eta_\sigma$, respectively. To model the mean of the logarithmic apartment prices ($\mu$), we define our linear predictor $\eta_\mu$ as 
\begin{equation*}
    \begin{split}
        \eta_\mu = b_{\mu} &+ \textstyle \sum_{j=1}^{13} x_j w_{\mu,j} + \textstyle \sum_{j=1}^4 f_{\mu,j}(z_j) + f_{\mu,5}(z_{5,1},z_{5,2})\\
        &+ d_{\mu,1}({u}_1) + d_{\mu,2}(u_2) + d_{\mu,3}(u_3), 
            \end{split}
\end{equation*}
with features given in Table~\ref{tab:feat}.
\begin{table}[tb]
    \centering
\begin{small}
\begin{tabular}{llll}
\textbf{Feature(s)} & \textbf{Description} && \textbf{Effect} \\ \hline
    $x_1,\ldots,x_3$ & room types && dummy-effect \\
    $x_4,\ldots,x_9$ & number of beds && dummy-effect \\
    $x_{10},\ldots,x_{13}$ & number of bedrooms && dummy-effect \\ \hline
    $z_1$ & accommodates & & thin-plate regression splines\\
    $z_2$& last review (in days)&&thin-plate regression splines\\
    $z_3$&reviews per month&&thin-plate regression splines\\
    $z_4$&review scores&&thin-plate regression splines\\
    $z_{5,1}, z_{5,2}$ & longitude, latitude && tensor product spline \\ \hline
    $u_1$ & images && CNN\\
    $u_2$ & description && embedding layer + FC \\
    $u_3$ & $x_1,\ldots,x_{13},z_1,\ldots,z_4$ && FC DNN (16-4-2)
\end{tabular}
\end{small}
    \caption{Overview of features in the data set, their description and how the features are parameterized. FC denotes a fully-connected layer with a single unit. The embedding layer is of size 100 for a lookup of 10,000 words with maximum sentence length of 100. The CNN architecture is described in the Supplementary Material~D.2. The FC DNN architecture consists of 16, 4 and 2 FC layers with dropout layers in between and is used to model the interactions of structured features.}
    \label{tab:feat}
\end{table}
For the scale parameter $\sigma$ we observe a better validation performance when only including a few structured predictors. As for the mean we include the room type as linear (dummy-)effects and a tensor product spline for location. Further model specifications and test results are given in Supplementary Material~D.
 \begin{figure}[tb]
 \vskip 0.2in
 \begin{center}
 \centerline{\includegraphics[trim = {3.5cm, 0, 3.5cm, 0}, width=0.63\columnwidth]{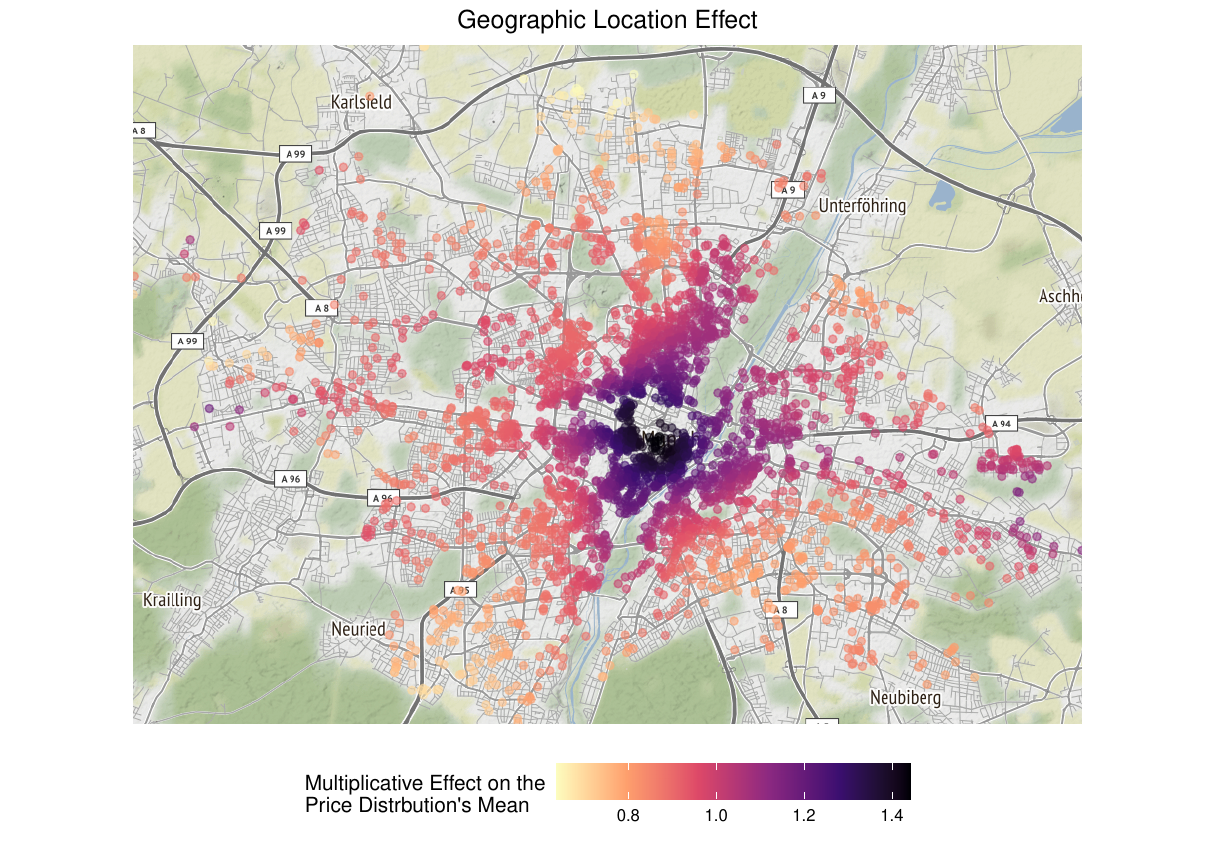}} 
 \caption{Estimated multiplicative effect (color) of the apartment's geographic location (points on the map) on the mean price value.}
 \label{munich}
 \end{center}
\end{figure}

\emph{Results}. Results suggest that structured effects are the main driving factor in our model. As images and descriptions are rather noisy in this data set, this result is not surprising. We observe a strong spatial effect of the geographic location on the log price of apartments. As shown in Figure~\ref{munich}, the more expensive apartments are located in the center of Munich, close to the English garden, or in rather prestigious areas. Relative to a typical apartment in Munich, the price for an apartment in the central part of Munich can, e.g., be up to 1.4 times more expensive, while a comparable apartment in, e.g., the south might cost only 60\% of the average room. 
The partial effect of the latent feature learned via the images and texts is only slightly correlated with the response. 
The multiplicative effect of remaining structured non-linear effects on the mean prices are visualized in Figure~\ref{peplots}. Here the number of accommodated people has the potentially highest impact with a tent-shaped partial multiplicative effect of over two times the average apartment price at $10$ accommodated people. The effects for $\eta_\sigma$ can be found in Table~2 and Figure~3 in the Supplementary Material. 
 \begin{figure}[tb]
 \vskip 0.2in
 \begin{center}
 \centerline{\includegraphics[width=0.7\columnwidth]{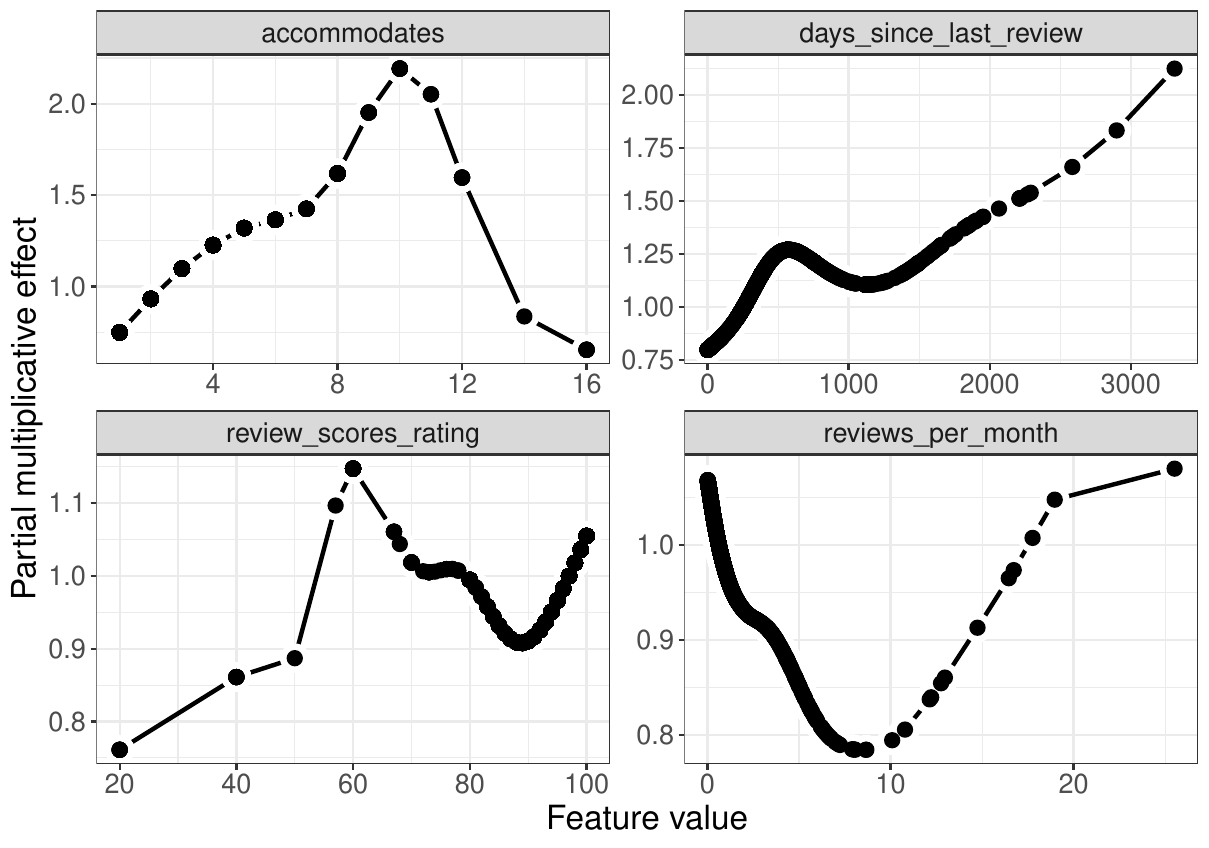}} 
 \caption{Estimated multiplicative effect (color) of  accommodates, days since last review, reviews per month and review rating, on the mean price value.}
 \label{peplots}
 \end{center}
\end{figure}

We also compare different model specifications of $\eta_\mu$ based on the correlation of predicted and true price values on train and test set. Table~\ref{tab:compOZ} compares the model with full information as given in Table~\ref{tab:feat} (Full), with and without orthogonalization (w/o OZ), with a model containing only structured effects (Structured), a model with structured effects and one DNN for the tabular data $u_3$ (Structured + DNN), a model with structured effects, DNN and text embedding (Structured + DNN + Texts) as well as a CNN where only the images are used (Images only). We here also find that the structured model part contributes the most, but adding the image information can help---however only when orthogonalized w.r.t. the structured effects. 

\begin{table}[ht]
\centering
\begin{tabular}{lrr}
  \hline
 Model & Train & Test \\ 
  \hline
 Structured & 0.68 & 0.64 \\ 
  Structured w/ DNN & 0.68 & 0.64 \\ 
  Full w/ OZ & 0.84 & \textbf{0.67} \\ 
  Full w/o OZ & 0.76 & 0.63 \\ 
  Structured + DNN + Texts & 0.68 & 0.64 \\ 
  Images only & 0.17 & 0.16 \\ 
   \hline
\end{tabular}
    \caption{Correlations of predicted and true values on train and test set (columns) for the six different models (rows).}
    \label{tab:compOZ}
\end{table}

\section{Conclusion and Outlook}\label{sec:conclusion}

In this work, we demonstrate why the combination of structured additive models and DNNs is not straightforward. 
We present a solution that provably permits the identification of structured effects in the presence of more flexible DNNs. We further develop a unified network architecture that combines SADR and DL by embedding the former into an overarching DNN. Using an orthogonalization cell enables the estimation of (interpretable) structured feature effects next to unstructured DNN predictors while ensuring identifiability of the structured model part(s). Simulations, benchmark studies, and a multimodal learning application demonstrate the generality and robustness of our proposed approach. 

While ensuring identifiability and interpretability of structured effects in the additive predictor, we note that the orthogonalization approach is much more versatile and can be applied to various other use cases beyond the identification of structured effects. In particular, using the orthogonalization, a structured effect can be ``subtracted'' from a latent effect learned from an unstructured data input, e.g., to adjust for confounders.

\section*{Acknowledgements}

We thank Almond St{\"o}cker and Dominik Thalmeier for their comments and helpful discussions. David R\"ugamer has been partly funded by the German Federal Ministry of Education and Research (BMBF) under Grant No. 01IS18036A. Nadja Klein acknowledges support by the Deutsche Forschungsgemeinschaft (DFG; German research foundation) through the Emmy Noether grant KL 3037/1-1.

\clearpage

\bigskip
\begin{center}
{\large\bf SUPPLEMENTARY MATERIAL}
\end{center}

\begin{description}

\item[Further Details:] The supplementary material includes proofs, algorithmic details as well as further specifications and results of numerical experiments.

\item[Reproducibility:] All codes used for this work are available at \url{https://github.com/davidruegamer/semi-structured_distributional_regression}.

\end{description}

\setlength{\bibsep}{0pt plus 0.3ex}
\bibliography{main}

\begin{thebibliography}{}

\bibitem[\protect\citeauthoryear{Campbell and Austin}{Campbell and
  Austin}{2002}]{Campbell02}
Campbell, J.~I. and S.~Austin (2002).
\newblock Effects of response time deadlines on adults' strategy choices for
  simple addition.
\newblock {\em Memory \& Cognition\/}~{\em 30\/}(6), 988--994.

\bibitem[\protect\citeauthoryear{Chi, Feltovich, and Glaser}{Chi
  et~al.}{1981}]{Chi81}
Chi, M.~T., P.~J. Feltovich, and R.~Glaser (1981).
\newblock Categorization and representation of physics problems by experts and
  novices.
\newblock {\em Cognitive science\/}~{\em 5\/}(2), 121--152.

\bibitem[\protect\citeauthoryear{Schubert, Denmark, Crandall, Grome, and
  Pappas}{Schubert et~al.}{2013}]{Schubert13}
Schubert, C.~C., T.~K. Denmark, B.~Crandall, A.~Grome, and J.~Pappas (2013).
\newblock Characterizing novice-expert differences in macrocognition: an
  exploratory study of cognitive work in the emergency department.
\newblock {\em Annals of emergency medicine\/}~{\em 61\/}(1), 96--109.

\end{thebibliography}


\begin{thebibliography}{}

\bibitem[\protect\citeauthoryear{Agarwal, Frosst, Zhang, Caruana, and
  Hinton}{Agarwal et~al.}{2020}]{Agarwal.2020}
Agarwal, R., N.~Frosst, X.~Zhang, R.~Caruana, and G.~E. Hinton (2020).
\newblock {Neural Additive Models: Interpretable Machine Learning with Neural
  Nets}.
\newblock {\em arXiv preprint arXiv:2004.13912\/}.

\bibitem[\protect\citeauthoryear{Arora, Cohen, Hu, and Luo}{Arora
  et~al.}{2019}]{Arora.2019}
Arora, S., N.~Cohen, W.~Hu, and Y.~Luo (2019).
\newblock Implicit regularization in deep matrix factorization.
\newblock pp.\  7411--7422.

\bibitem[\protect\citeauthoryear{Bishop}{Bishop}{1994}]{Bishop.1994}
Bishop, C.~M. (1994).
\newblock Mixture density networks.

\bibitem[\protect\citeauthoryear{Br{\'a}s-Geraldes, Papoila, and
  Xufre}{Br{\'a}s-Geraldes et~al.}{2019}]{Bras.2019}
Br{\'a}s-Geraldes, C., A.~Papoila, and P.~Xufre (2019).
\newblock Generalized additive neural network with flexible parametric link
  function: model estimation using simulated and real clinical data.
\newblock {\em Neural Computing and Applications\/}~{\em 31\/}(3), 719--736.

\bibitem[\protect\citeauthoryear{Buja, Hastie, and Tibshirani}{Buja
  et~al.}{1989}]{Buja.1989}
Buja, A., T.~Hastie, and R.~Tibshirani (1989).
\newblock Linear smoothers and additive models.
\newblock {\em The Annals of Statistics\/}~{\em 17\/}(2), 453--510.

\bibitem[\protect\citeauthoryear{Chen, Doddi, Royer, Freschi, Schito, Ezewudo,
  Kohane, Beam, and Farhat}{Chen et~al.}{2018}]{Chen.2018}
Chen, M.~L., A.~Doddi, J.~Royer, L.~Freschi, M.~Schito, M.~Ezewudo, I.~S.
  Kohane, A.~Beam, and M.~Farhat (2018).
\newblock Deep learning predicts tuberculosis drug resistance status from
  whole-genome sequencing data.
\newblock {\em BioRxiv\/}, 275628.

\bibitem[\protect\citeauthoryear{Cheng, Koc, Harmsen, Shaked, Chandra, Aradhye,
  Anderson, Corrado, Chai, Ispir, et~al.}{Cheng et~al.}{2016}]{Cheng.2016}
Cheng, H.-T., L.~Koc, J.~Harmsen, T.~Shaked, T.~Chandra, H.~Aradhye,
  G.~Anderson, G.~Corrado, W.~Chai, M.~Ispir, et~al. (2016).
\newblock Wide \& deep learning for recommender systems.
\newblock pp.\  7--10.

\bibitem[\protect\citeauthoryear{Cornwell and Rupert}{Cornwell and
  Rupert}{1988}]{Cornwell.1988}
Cornwell, C. and P.~Rupert (1988).
\newblock Efficient estimation with panel data: An empirical comparison of
  instrumental variables estimators.
\newblock {\em Journal of Applied Econometrics\/}~{\em 3\/}(2), 149--155.

\bibitem[\protect\citeauthoryear{Cybenko}{Cybenko}{1989}]{Cybenko.1989}
Cybenko, G. (1989).
\newblock Approximation by superpositions of a sigmoidal function.
\newblock {\em Mathematics of control, signals and systems\/}~{\em 2\/}(4),
  303--314.

\bibitem[\protect\citeauthoryear{De~Waal and Du~Toit}{De~Waal and
  Du~Toit}{2011}]{De.2011}
De~Waal, D.~A. and J.~V. Du~Toit (2011).
\newblock Automation of generalized additive neural networks for predictive
  data mining.
\newblock {\em Applied Artificial Intelligence\/}~{\em 25\/}(5), 380--425.

\bibitem[\protect\citeauthoryear{Erichson, Voronin, Brunton, and Kutz}{Erichson
  et~al.}{2019}]{Erichson.2016}
Erichson, B., S.~Voronin, S.~Brunton, and J.~N. Kutz (2019).
\newblock Randomized matrix decompositions using {R}.
\newblock {\em Journal of Statistical Software\/}~{\em 89\/}(11).

\bibitem[\protect\citeauthoryear{Groll, Hambuckers, Kneib, and Umlauf}{Groll
  et~al.}{2019}]{Groll.2019}
Groll, A., J.~Hambuckers, T.~Kneib, and N.~Umlauf (2019).
\newblock Lasso-type penalization in the framework of generalized additive
  models for location, scale and shape.
\newblock {\em Computational Statistics \& Data Analysis\/}~{\em 140}, 59--73.

\bibitem[\protect\citeauthoryear{Hubin, Storvik, and Frommlet}{Hubin
  et~al.}{2018}]{Hubin.2018}
Hubin, A., G.~Storvik, and F.~Frommlet (2018).
\newblock Deep bayesian regression models.

\bibitem[\protect\citeauthoryear{Kingma and Ba}{Kingma and
  Ba}{2014}]{kingma2014adam}
Kingma, D.~P. and J.~Ba (2014).
\newblock Adam: A method for stochastic optimization.
\newblock {\em arXiv preprint arXiv:1412.6980\/}.

\bibitem[\protect\citeauthoryear{Klein, Kneib, Lang, and Sohn}{Klein
  et~al.}{2015}]{Klein.2015}
Klein, N., T.~Kneib, S.~Lang, and A.~Sohn (2015).
\newblock Bayesian structured additive distributional regression with an
  application to regional income inequality in {G}ermany.
\newblock {\em Annals of Applied Statistics\/}~{\em 9\/}(2), 1024--1052.

\bibitem[\protect\citeauthoryear{Koenker}{Koenker}{2005}]{Koenker.2005}
Koenker, R. (2005).
\newblock {\em Quantile Regression}, Volume Economic Society Monographs.
\newblock Cambridge University Press.

\bibitem[\protect\citeauthoryear{Kuleshov, Fenner, and Ermon}{Kuleshov
  et~al.}{2018}]{kuleshov18a}
Kuleshov, V., N.~Fenner, and S.~Ermon (2018).
\newblock Accurate uncertainties for deep learning using calibrated regression.
\newblock ~{\em 80}, 2796--2804.

\bibitem[\protect\citeauthoryear{Li, Reich, and Bondell}{Li
  et~al.}{2021}]{Li.2021}
Li, R., B.~J. Reich, and H.~D. Bondell (2021).
\newblock Deep distribution regression.
\newblock {\em Computational Statistics \& Data Analysis\/}~{\em 159}, 107203.

\bibitem[\protect\citeauthoryear{Mayr, Fenske, Hofner, Kneib, and Schmid}{Mayr
  et~al.}{2012}]{Mayr.2012}
Mayr, A., N.~Fenske, B.~Hofner, T.~Kneib, and M.~Schmid (2012).
\newblock Generalized additive models for location, scale and shape for
  high-dimensional data - a flexible approach based on boosting.
\newblock {\em Journal of the Royal Statistical Society, Series C - Applied
  Statistics\/}~{\em 61\/}(3), 403--427.

\bibitem[\protect\citeauthoryear{Nelder and Wedderburn}{Nelder and
  Wedderburn}{1972}]{Nelder.1972}
Nelder, J.~A. and R.~W. Wedderburn (1972).
\newblock Generalized linear models.
\newblock {\em Journal of the Royal Statistical Society: Series A
  (General)\/}~{\em 135\/}(3), 370--384.

\bibitem[\protect\citeauthoryear{Ong, Nott, and Smith}{Ong
  et~al.}{2018}]{Ong.2018}
Ong, V. M.-H., D.~J. Nott, and M.~S. Smith (2018).
\newblock Gaussian variational approximation with a factor covariance
  structure.
\newblock {\em Journal of Computational and Graphical Statistics\/}~{\em
  27\/}(3), 465--478.

\bibitem[\protect\citeauthoryear{Pölsterl, Sarasua, Gutiérrez-Becker, and
  Wachinger}{Pölsterl et~al.}{2020}]{Poelsterl.2020}
Pölsterl, S., I.~Sarasua, B.~Gutiérrez-Becker, and C.~Wachinger (2020).
\newblock A wide and deep neural network for survival analysis from anatomical
  shape and tabular clinical data.
\newblock {\em Communications in Computer and Information Science\/},
  453–464.

\bibitem[\protect\citeauthoryear{Rigby and Stasinopoulos}{Rigby and
  Stasinopoulos}{2005}]{Rigby.2005}
Rigby, R.~A. and D.~M. Stasinopoulos (2005).
\newblock Generalized additive models for location, scale and shape.
\newblock {\em Journal of the Royal Statistical Society: Series C (Applied
  Statistics)\/}~{\em 54\/}(3), 507--554.

\bibitem[\protect\citeauthoryear{Roberts and Roberts}{Roberts and
  Roberts}{2020}]{Roberts.2020}
Roberts, D. A.~O. and L.~R. Roberts (2020).
\newblock Qr and lq decomposition matrix backpropagation algorithms for square,
  wide, and deep -- real or complex -- matrices and their software
  implementation.

\bibitem[\protect\citeauthoryear{{Rodrigues} and {Pereira}}{{Rodrigues} and
  {Pereira}}{2020}]{Rodrigues.2018}
{Rodrigues}, F. and F.~C. {Pereira} (2020).
\newblock Beyond expectation: Deep joint mean and quantile regression for
  spatiotemporal problems.
\newblock {\em IEEE Transactions on Neural Networks and Learning Systems\/},
  1--13.

\bibitem[\protect\citeauthoryear{R{\"u}gamer, Brockhaus, Gentsch, Scherer, and
  Greven}{R{\"u}gamer et~al.}{2018}]{Ruegamer.2018}
R{\"u}gamer, D., S.~Brockhaus, K.~Gentsch, K.~Scherer, and S.~Greven (2018).
\newblock Boosting factor-specific functional historical models for the
  detection of synchronization in bioelectrical signals.
\newblock {\em Journal of the Royal Statistical Society: Series C (Applied
  Statistics)\/}~{\em 67\/}(3), 621--642.

\bibitem[\protect\citeauthoryear{Ruppert, Wand, and Carroll}{Ruppert
  et~al.}{2003}]{Ruppert.2003}
Ruppert, D., M.~P. Wand, and R.~J. Carroll (2003).
\newblock {\em Semiparametric regression}.
\newblock Cambridge and New York: Cambridge University Press.

\bibitem[\protect\citeauthoryear{Sarle}{Sarle}{1994}]{Sarle.1994}
Sarle, W.~S. (1994).
\newblock Neural networks and statistical models.

\bibitem[\protect\citeauthoryear{Silverman}{Silverman}{1985}]{Silverman.1985}
Silverman, B.~W. (1985).
\newblock Some aspects of the spline smoothing approach to non-parametric
  regression curve fitting.
\newblock {\em Journal of the Royal Statistical Society: Series B
  (Methodology)\/}~{\em 47\/}(1), 1--21.

\bibitem[\protect\citeauthoryear{Song, Diethe, Kull, and Flach}{Song
  et~al.}{2019}]{Song.2019}
Song, H., T.~Diethe, M.~Kull, and P.~Flach (2019).
\newblock Distribution calibration for regression.
\newblock ~{\em 97}, 5897--5906.

\bibitem[\protect\citeauthoryear{{Tran}, {Nguyen}, {Nott}, and {Kohn}}{{Tran}
  et~al.}{2020}]{Tran.2018}
{Tran}, M.-N., N.~{Nguyen}, D.~{Nott}, and R.~{Kohn} (2020).
\newblock Bayesian deep net {GLM} and {GLMM}.
\newblock {\em Journal of Computational and Graphical Statistics\/}~{\em
  29\/}(1), 97--113.

\bibitem[\protect\citeauthoryear{Umlauf, Klein, and Zeileis}{Umlauf
  et~al.}{2018}]{Umlauf.2018}
Umlauf, N., N.~Klein, and A.~Zeileis (2018).
\newblock Bamlss: Bayesian additive models for location, scale, and shape (and
  beyond).
\newblock {\em Journal of Computational and Graphical Statistics\/}~{\em
  27\/}(3), 612--627.

\bibitem[\protect\citeauthoryear{Wood}{Wood}{2017}]{Wood.2017}
Wood, S.~N. (2017).
\newblock {\em Generalized additive models: an introduction with R}.
\newblock Chapman and Hall/CRC.

\end{thebibliography}
\bibliographystyle{chicago}

\clearpage
\appendix
\onecolumn

\section{Proofs} \label{sec:proofs}

In the following, we provide proofs for our main results from Section~\ref{sec:ortho}.

\subsection{General Setting}

In the main part of the paper we assume w.l.o.g.~only one linear effect and one  DNN to be present in each model predictor $\eta_k$. However, both Lemma~\ref{lem:1} and Theorem~\ref{theo1}  hold in more general settings, since
\begin{enumerate}
\item multiple non-linear effects $f_{k,j}$ can be always represented as linear combinations $\bm{Z}\bm{\delta}_k$ with $\bm{Z} = (\bm{z}_1^\top, \ldots, \bm{z}_r^\top)^\top$ \citep[see, e.g.,][]{Wood.2017}. Thus an analogous proof holds when replacing $\bm{X}$ with the composed matrix $\widetilde{\bm{X}} := (\bm{X}|\bm{Z}) \in \mathbb{R}^{n\times (p+r)}$ and $\bm{\mathcal{P}}^\bot_X$ accordingly with $\bm{\mathcal{P}}^\bot_{\widetilde{X}}$;
\item further DNN predictors can be orthogonalized individually as the outputs of different DNNs $d_{k,1}, \ldots, d_{k,g_k}$ are combined as $\hat{\bm{U}}_{k,1}\bm{\gamma}_{k,1} + \ldots + \hat{\bm{U}}_{k,g_k}\bm{\gamma}_{k,g_k}$, allowing individual multiplication with $\bm{\mathcal{P}}_X$ (or $\bm{\mathcal{P}}^\bot_{\widetilde{X}}$) from the left;
\item identifiability when mixing several structured effects and several DNNs each with different inputs can be ensured by individually orthogonalizing each DNN output $\hat{\bm{U}}_{k,j}, j=1,\ldots,g_k,$ with respect to the union  of those features that intersect with any structured effect.  
\end{enumerate}

\subsection{Proof of Lemma~\ref{lem:1} (Orthogonalization)}

\begin{proof}
First decompose the predictor $\bm{\eta}_k$ into $\bm{\mathcal{P}}_X \bm{\eta}_k + \bm{\mathcal{P}}_X^\bot \bm{\eta}_k$. Plugging this decomposition into the definition of \eqref{thetak} shows that in the case when 
\\
a) the true linear effect of $\bm{X}$ is zero, it must hold that $\bm{\eta}_k = \bm{\mathcal{P}}^\bot_X \bm{\eta}_k$ and thus $\bm{\eta}_k = \bm{\mathcal{P}}^\bot_X (\bm{X}{\bm{w}}_k + \widetilde{\bm{U}}_k \bm{\gamma}_k) = \bm{0}_{n \times 1} + \widetilde{\bm{U}}_k \bm{\gamma}_k$,
\\
b) no unstructured effect is present, $\bm{\eta}_k = \bm{\mathcal{P}}_X \bm{\eta}_k$ and thus $\bm{\eta}_k = \bm{\mathcal{P}}_X(\bm{X}{\bm{w}}_k + \bm{\mathcal{P}}_X^\bot \widehat{\bm{U}}_k \bm{\gamma}_k) = \bm{X}\bm{w}_k + \bm{0}_{n \times 1}$ and c) both effect types are present multiplying both sides of \eqref{thetak} with either $\bm{\mathcal{P}}_X$ or $\bm{\mathcal{P}}_X^\bot$ yields the desired property.  
\end{proof}

\subsection{Proof of Theorem~\ref{theo1} (Identifiability)}

\begin{proof} 
Assume there exists an identifiability issue of linear effects of $\bm{X}$, i.e., there exists $ \bm{\xi}_k: \bm{\mathcal{P}}_X \bm{\eta}_k = (\bm{X}\bm{w}_k - \bm{\xi}_k) + \bm{\xi}_k$ for at least one $k\in\lbrace 1,\ldots,K\rbrace$ with $\bm{\mathcal{P}}_X \widetilde{\bm{U}}_k\bm{\gamma}_k =: \bm{\xi}_k \neq \bm{0}_{n \times 1}$. However, from the decomposition in Lemma~\ref{lem:1} it directly follows that $\xi_k \equiv  \bm{0}_{n \times 1}$.
\end{proof}

\section{Further Details on Simulation Study}\label{app:sim}

\subsection{Details on Section~\ref{sec:simident}}

The functions used to generate non-linear relationships are:
\begin{itemize}
    \item $f_1(x) = \cos(5x)$,
    \item $f_2(x) = \tanh(3x)$,
    \item $f_3(x) = -x^3$,
    \item $f_4(x) = \cos(3x-2) \cdot (-3x)$,
    \item $f_5(x) = \exp(2x) -1$,
    \item $f_6(x) = x^2$,
    \item $f_7(x) = \sin(x)\cos(x)$,
    \item $f_8(x) = \sqrt(|x|)$,
    \item $f_9(x) = -x^5$,
    \item $f_{10}(x) = \log(x^2)/100$.
\end{itemize}
To ensure identifiability in the data generating process also, we remove linear trends in these functions prior to generating the outcome using the additive predictor.

\subsection{Details on Section~\ref{sec:simcomp}}

We generate all features using a standard uniform distribution. The coefficients of linear effects are drawn from a uniform distribution with limits $-3$ and $3$. We use the non-linear functions as defined in the previous subsection and additionally use the following 3 functions in cases where more than 10 non-linear functions are used (due to the multiple distribution parameters and depending on the amount of overlap between the predictors in these parameters):
\begin{itemize}
    \item $f_{11}(x) = \sin(10x)$,
    \item $f_{12}(x) = (0.1\sin(10x) + 1) \cdot I(x<0) + (-2x + 1) \cdot I(x \geq 0)$,
    \item $f_{13}(x) = -x \cdot \tanh(3x) \cdot \sin(4x)$.
\end{itemize}
As intercepts for the location we use $1$, for the scale $-1$. After adding up intercept, linear effects and non-linear effects, the additive predictors $\eta_k, k=1,2$ are transformed according to the link function of the corresponding distribution and distribution parameter. Using the resulting distribution parameters $\theta_k$, the outcome values are drawn from the distribution defined by these parameters.

\section{Benchmark Studies: Further Experiments}\label{app:benchmark}

In the following, we provide further benchmark experiments on the use of our framework for quantile regression (Section~\ref{sec:deepdist}), calibrated regression
(Section~\ref{app:benchmark:1}) and high-dimensional classification (Section~\ref{sec:high}).

\subsection{Deep Distributional Models for Quantiles} \label{sec:deepdist}

To illustrate the distributional aspect of our framework, we consider the motorcycle data set from \citet{Silverman.1985} and reanalyzed in a distributional context in \citet{Rodrigues.2018}. In contrast to \citet{Rodrigues.2018}, who present a framework to jointly predict quantiles, our approach models the entire distribution as $\mathcal{N}(d_\mu(time), \exp(b + f(time)))$, yielding predictions for all quantiles in $(0,1)$ in one single model. Here, $d_{\mu}$ is defined as a NN with two hidden layers, tanh and linear activation, and 50 and 10 hidden units, respectively. Optimization details are the same as in Subsection~\ref{subsec:deepmixed}.

\emph{Results}: As we model the distribution itself and not the quantiles explicitly, our approach does not suffer from quantile crossings. Using the quantiles $0.1,0.4,0.6$ and $0.9$, the approach by \citet{Rodrigues.2018} yields an RMSE of 0.526 (0.003) with an average of 3.050 (0.999) quantile crossings on all test data points. In contrast, our approach with DNN for the mean and linear time effect for the distribution's scale 
exhibits no quantile crossings and yields an RMSE of 0.536 (0.016). 
 Figure~\ref{motor} visualizes the results for the analysis of the motorcycle data from SSDR with estimated mean curve and four different quantiles inferred from the estimated distribution.
 \begin{figure}[htbp]
 \vskip 0.2in
 \begin{center}
 \centerline{\includegraphics[width=0.7\columnwidth]{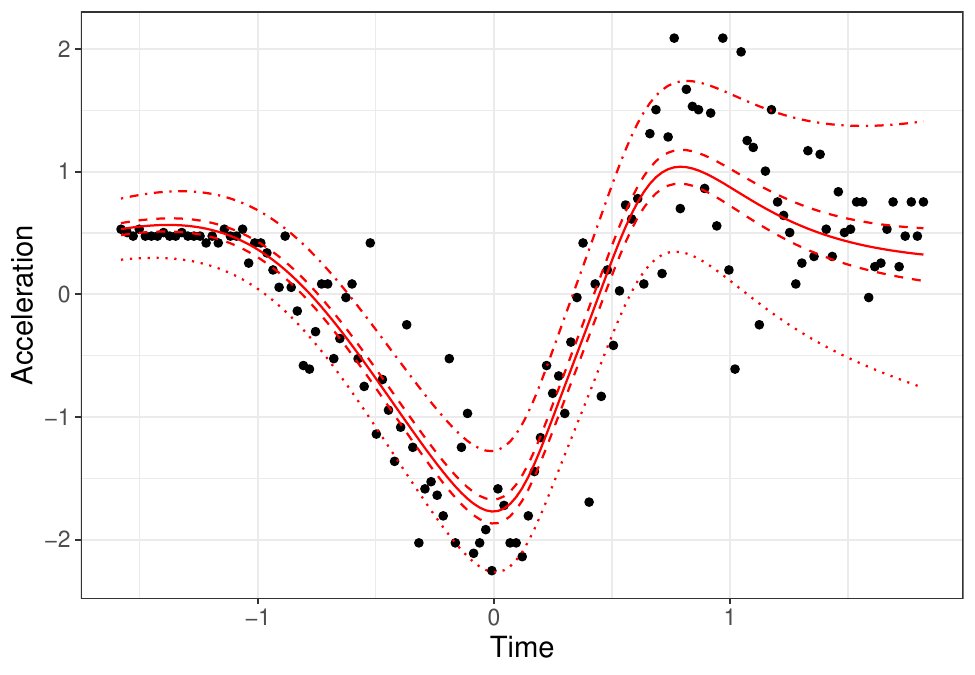}} 
 \caption{Acceleration data over time for motorcycle data, with estimated mean (solid line), 40\%-, 60\%-quantiles (dashed lines) and 10\%- as well as 90\%-quantiles (dashed-dotted line) in red.}
 \label{motor}
 \end{center}
\vskip -0.2in
\end{figure}

\subsection{Deep Calibrated Regression}\label{app:benchmark:1}

SSDR assumes a normal distribution $\mathcal{N}({\mu}, {\sigma}^2)$ for the output with parameters $\mu$ and $\sigma^2$. We denote their respective DNN predictors by $d_\mu(\cdot)$ and $d_\sigma(\cdot)$, with the number of hidden units in brackets. Specifically for the four data sets, we model $\mu,\,\sigma^2$ as $b_\mu + f(z_3) + d_\mu(4)$, $\exp(b_\sigma + d_\sigma(4))$ (Diabetes); $b_\mu+\bm{x}^\top\bm{w}+\sum_{j=1}^{11} f_j(z_j)+d_\mu(32,16,4)$, $\exp(b_\sigma + d_\sigma(2))$, where $\bm{x}=(z_1,\ldots,z_J)^\top, J=13$ (Boston);  ${b_\mu + \sum_{j\in J} f_j(z_j) + d_\mu(16,4)}$, $\exp(b_\sigma + d_\sigma(16,4))$, where $J=\{1,2,5\}$ being the indices for three of five available numerical features (Airfoil); and $b_\mu + d_\mu(16,4)$, $ \exp(b_\sigma + \sum_j w_{j})$ with index $j \in \{1,\ldots,12\}$ for each month (Forest Fire). For the DNNs, we use tanh activation functions in the hidden layers with 16 and 4 units followed by a single unit layer with linear activation. Based on early experiments we altered this default for Boston by making the DNN for the mean more complex (with architecture 32-16-4-1) while reducing the complexity of the DNN for the scale (2-1). For Diabetes we also reduced the complexity of the DNNs for both parameters by removing the layer with 16 units. Adam optimizer with a learning rate default of $0.001$ and a batch size of 32 was employed to fit the models, with the number of epochs chosen by 5-fold cross-validation.

\subsection{High-Dimensional Classification} \label{sec:high}

\citet{Ong.2018} aim at predicting various forms of cancer in high-dimensional gene expression data sets (Colon, Leukaemia, and Breast cancer). The authors propose VAFC, a Bayesian approach utilizing horseshoe priors and VI. We implement SSDR using a Bernoulli distribution and combine a linear model with a small DNN (either one or two hidden layers, ReLu activations, and up to 16 hidden units). The model is applied to all three data sets with training sample sizes of 42, 38, and 38 and test set sizes of 20, 34, and 4, respectively.  The number of features is $p=2000$ (Colon) and $p=7129$ (Leukaemia, Breast). As an additional comparison, we fit a standard DNN (sigmoid activation function and binary cross-entropy loss) with the same architecture as for the SSDR approach, but no additional structured predictors. SSDR and DNN models are optimized with Adam using a learning rate of $0.001$ and the number of epochs chosen by cross-validation.

\emph{Results}. Table~\ref{cancer} compares the average (standard deviation) of the Area under the Receiver Operator Characteristic Curve (AUROC). While all approaches yield an AUROC of one on the Colon cancer data, our SSDR approach is able to outperform the VAFC and standard DNN approach on the other two data sets. 

\begin{table}[htbp]
\caption{Comparison of AUROC on three cancer data sets (first row: Colon cancer; second row: Leukaemia; thrid row: Breast cancer) for our method, a simple DNN and the VAFC (with different number of factors)}
\label{cancer}
   \vskip 0.15in
\begin{center}
\begin{tabular}{cccc}
  SSDR & DNN & VAFC (4) & VAFC (20) \\\hline
    1.00 (0.00) & 1.00 (0.00) &   1.00 (0.00) & 1.00 (0.00) \\
  \textbf{0.98} (0.02) & 0.82 (0.23)  & 0.91 (0.06) & 0.90 (0.07)\\
\textbf{1.00} (0.00) & \textbf{1.00} (0.00)   & 0.95 (0.10) & 0.84 (0.12) \\
\hline
\end{tabular}
\end{center}
\vskip -0.2in
\end{table}

\section{Application: Further Details} \label{app:application}

In the following, we provide further details on the presented Airbnb application. We briefly describe the training of our framework, further results of the proposed model, and the CNN architecture used in the models with images in the unstructured predictor part. All specifications are trained with Adam for $4000$ epochs with early stopping, using a learning rate of $0.001$ and batch size of 32.

\subsection{Warm Starts}

In the application, we first estimate a structured model without any DNN and use the estimated coefficients as a warm start for the multimodal network. This effectively tackles the problem of different learning speeds of the two model parts without the need for a dedicated optimization routine.

\subsection{CNN Architecture}\label{app:cnn}

The trunk of our CNN consists of 3 convolution blocks (filters = 32, 64, 128, kernel sizes = (3,3), ReLU activation, batch normalization, max-pooling 2D with pool sizes = (3,3), (2,2), (2,2) and dropout layers with rate 0.25) followed by a flatten layer, an FC layer with 128 units, a ReLU activation, batch normalization, a layer with 20\% dropout and a FC layer with 64 units. The network head is a 1-unit hidden layer with linear activation to allow for orthogonalization.

\subsection{Further Results}

Figure~\ref{image_munich} visualizes the density of the partial effect $\widetilde{\bm{U}}_1 \hat{\gamma}_1$ of images on the mean apartment prices.
 \begin{figure}[ht]
 \vskip 0.2in
 \begin{center}
 \centerline{\includegraphics[width=\columnwidth]{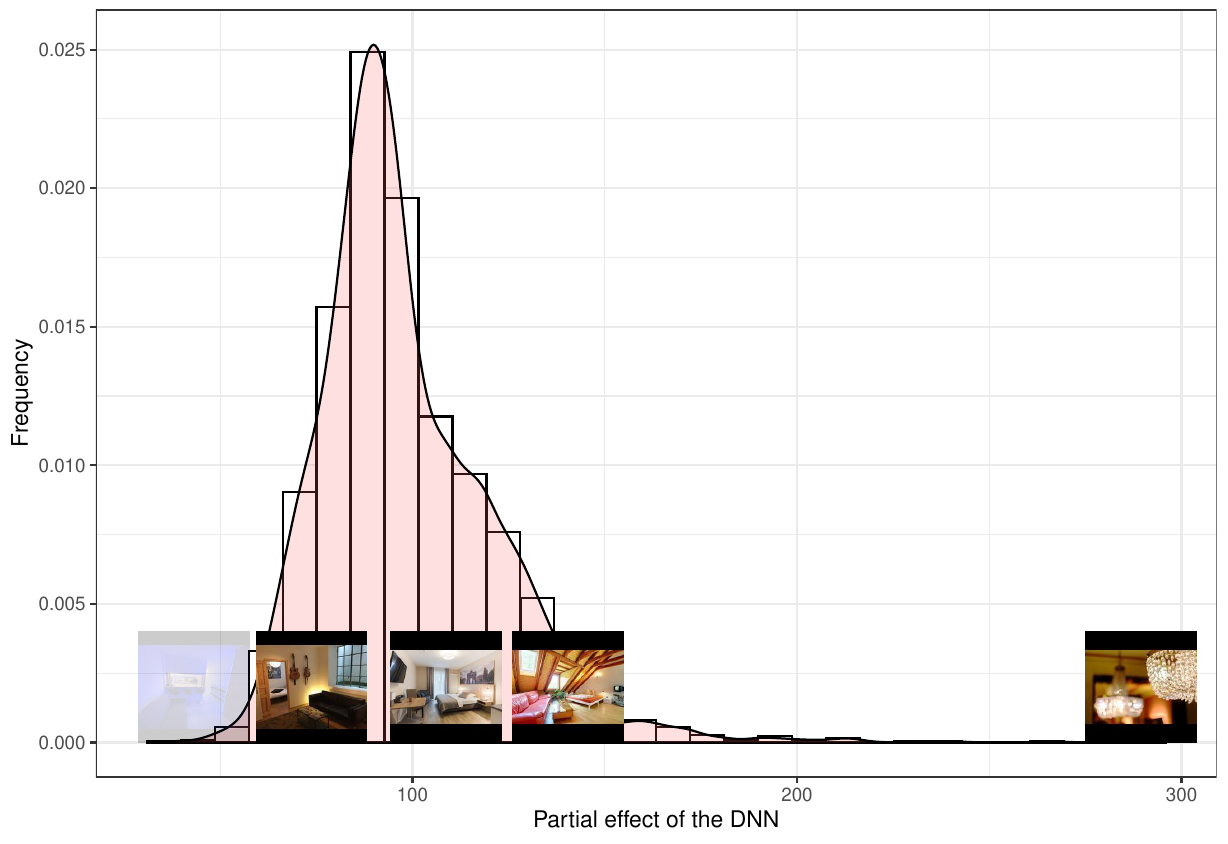}} 
 \caption{Distribution of partial effects of the image DNN for the mean apartment price with five selected examples.}
 \label{image_munich}
 \end{center}
\vskip -0.2in
\end{figure}
The estimated coefficients for the structured linear predictors are given in the following Table~\ref{tab:struct_pred} with no entries if the effect was not included in the predictor. The reference apartment in the intercept corresponds to $0$ bedrooms and beds (which means either the entry is missing or there is no dedicated bedroom, but instead, e.g., a dormitory) and an entire home or apartment as room type. 
\begin{table}[ht]
\centering
\begin{tabular}{r|cc}
  \hline
 & $\eta_\mu$ & $\exp(\eta_\sigma)$  \\ 
  \hline
Intercept & 37.46 & 2.09 \\
  room type = Hotel room & 1.29 & 2.53 \\ 
  room type = Private room & 0.74 & 2.51 \\ 
  room type = Shared room & 0.78 & 5.21 \\ 
  bedrooms = 1 & 1.12 & - \\ 
  bedrooms = 2 & 1.42 & - \\ 
  bedrooms = 3 & 1.77 & - \\ 
  bedrooms = 4 & 2.41 & - \\ 
  beds = 1 & 1.27 & - \\ 
  beds = 2 & 1.34 & - \\ 
  beds = 3 & 1.33 & - \\ 
  beds = 4 & 1.28 & - \\ 
  beds = 5 & 1.61 & - \\ 
  beds = 6 & 1.25 & - \\ 
   \hline
\end{tabular}
\caption{Partial multiplicative effect of structured linear predictors for the location ($\eta_\mu$) and scale ($\eta_\sigma$)}
\label{tab:struct_pred}
\end{table}

The effect of the geographic location on the estimated distribution scale is depicted in Figure~\ref{munich_scale}.
 \begin{figure}[ht]
 \vskip 0.2in
 \begin{center}
 \centerline{\includegraphics[trim = {3.5cm, 0, 3.5cm, 0}, width=0.6\columnwidth]{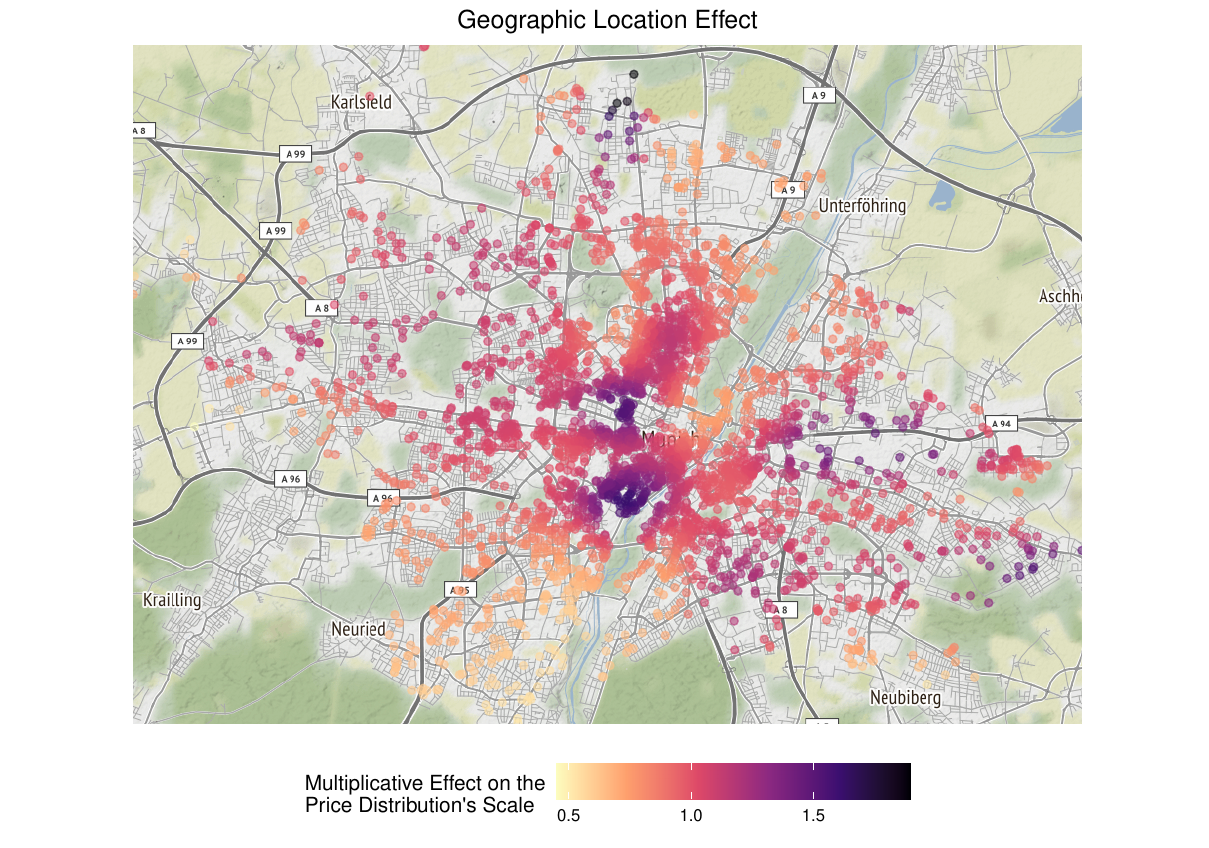}} 
 \caption{Estimated multiplicative effect (color) of the apartment's geographic location (points on the map) on the scale of the price value distribution.}
 \label{munich_scale}
 \end{center}
\vskip -0.2in
\end{figure}
Similar to the effect on the mean, the scale parameter is larger in the center of Munich, but also in some less prestigious areas. The effect of these locations is yet much more pronounced than the effect on the mean.

\section{Details on Computing Environment}

Benchmarks and simulation studies were performed on a Linux server with 32 CPUs (Intel(R) Xeon(R) CPU E5-2650 v2 @ 2.60GHz), 64 GB RAM and took several minutes (Benchmarks, Identifiability Study) to 4 days (Model Comparison Study). The application was performed on a personal computer with 4 CPUs (Intel(R) Core(TM) i7-8665U CPU @ 1.90GHz), 16 GB RAM and took 6.8 hours.

\end{document}